\documentclass[lettersize,journal]{IEEEtran}
\usepackage{amsmath,amsfonts}
\usepackage{algorithmic}
\usepackage{algorithm}
\usepackage{array}
\usepackage[caption=false,font=normalsize,labelfont=sf,textfont=sf]{subfig}
\usepackage{textcomp}
\usepackage{stfloats}
\usepackage{url}
\usepackage{verbatim}
\usepackage{graphicx}
\usepackage{cite}
\usepackage{xcolor}
\usepackage{graphicx} 
\usepackage{booktabs} 
\usepackage{booktabs,graphicx}
\DeclareMathOperator*{\argmin}{arg\,min}
\DeclareMathOperator*{\argmax}{arg\,max}
\newtheorem{lemma}{Lemma}
\newtheorem{corollary}{Corollary} 
\newtheorem{theorem}{Theorem} 
\newcommand{\Wtwoone}{W_2^{(1)}}
\newcommand{\Wtwotwo}{W_2^{(2)}}

\begin{document}

\title{Risk and Anomaly Identification  \\ 
	for Distribution Network Optimal Operation   \\
	Based on Reinforcement Learning \\
	and Uncertainty Quantification}
\author{Ziqi Zhang,~\IEEEmembership{Member,~IEEE} 
	%
\thanks{
	Z. Zhang is with the College of Automation Engineering, Nanjing University of Aeronautics and Astronautics, Nanjing 211100, China
	(E-mail: 1124562662@qq.com ). }
	%
 
}



\maketitle

\begin{abstract}
	
Reliable operation of modern distribution networks requires timely identification of operational risks and anomalous events under pervasive uncertainty. In practice, operators must identify risks that are inherent in stochastic yet in-distribution conditions, and anomalies that correspond to out-of-distribution behaviors such as unusual load patterns, extreme weather or cyber-physical attacks. This paper addresses this joint risk and anomaly identification problem for optimal distribution network operation and proposes a deep reinforcement learning framework that is explicitly uncertainty aware. We integrate distributional and Bayesian deep reinforcement learning to realize a second-order uncertainty quantification scheme that decomposes total uncertainty into aleatoric and epistemic components, which are respectively used to characterize inherent risk and out-of-distribution anomalies. The resulting epistemic estimates drive both exploration during training and out-of-distribution detection with fallback control during deployment, whereas aleatoric estimates are used to characterize intrinsic operational risk. Simulation results demonstrate the 
performance of our DRL agent and the effectiveness of the uncertainty quantification. 
	  
\end{abstract}

\begin{IEEEkeywords}
Anomaly detection, deep reinforcement learning, 
distribution network operation, energy management systems, 
uncertainty quantification. 
\end{IEEEkeywords}

\section{Introduction}
\label{sec:intro}

\IEEEPARstart{M}{odern} distribution network (DN) operation is increasingly shaped by renewable energy sources (RES) such as solar photovoltaics (PV), distributed energy storage systems (ESS) and distributed diesel generators (DGs). These factors make operation stochastic, high-dimensional, and time-coupled. Model-free deep reinforcement learning (DRL) \cite{ref5} gives real-time policy for nonlinear DNs.
However, despite DRL's advantages, vanilla DRL techniques face trustworthiness issues when applied to real-world DNs, namely lacking the ability of uncertainty quantification (UQ). In the realm of UQ, uncertainties are generally classified into two categories: (1) epistemic uncertainty (EU) arising from a lack of knowledge or information about the current situation; (2) aleatoric uncertainty (AU) inherent to the environment and cannot be reduced by gathering more information. The combination of the two is termed as total uncertainty.

For a DRL agent, high EU may be caused by out-of-distribution (OOD) scenarios, where the operational environment deviates significantly from training data, contrasting to in-distribution (InD) scenarios where the operational environment remains close to training data. Such distributional shifts may arise from various DN anomalies, encompassing: (1) atypical load profiles, such as variable charging patterns of electric vehicles, air conditioning demand surge due to heatwaves and cold snaps, or spikes because of unplanned public events, all of which introduce complex and non-stationary patterns that are challenging to detect; (2) unexpected RES behaviors, such as storm-induced damage, inverter malfunctions, or irregular cloud shading over PV systems, where spatio-temporal dynamics of cloud movements create a combinatorial explosion of non-uniform shading patterns on distributed PV arrays \cite{ref6}; (3) the deterioration of infrastructures, such as slow degradation of transformer insulation, corrosion of conductors and connectors, and increased impedance in aging cables, which collectively result in small but cumulative shifts in the whole system's physical parameters, often indistinguishable from random noises.

In these OOD scenarios with conditions not seen during training, DRL agents based on neural networks (NNs) might fail to generalize robustly \cite{ref7}. This can result in confident yet suboptimal or unsafe actions, such as improper voltage regulation, excessive reactive power compensation or inefficient diesel generator dispatch. These actions can lead to inefficient energy management, branch overloading, violation of nodal voltage limits, severe load imbalances, potentially causing increased operational costs, widespread power outages, or even physical damage to critical grid components.

The situation can be exacerbated by limited availability of the training data, especially in DNs due to privacy concerns (e.g., smart meters and electric vehicles' privacy) or infrastructure constraints (e.g., sparse monitoring devices). What is worse, the combination of the time-series data and the large number of DN buses contributes to the high dimensionality of the scenario, bringing the curse-of-dimensionality phenomenon that most points tend to have similar distances among them \cite{ref8}. Such phenomenon, along with the complex interpolation and extrapolation of NNs, may render traditional OOD detection approaches such as those based on metrics or density estimation \cite{ref9} less effective. A related challenge to the OOD problem is the presence of adversarial examples \cite{ref10}. These examples are much closer to the training data distribution, with their differences from certain training samples often resembling small perturbations. Despite these subtle differences, adversarial examples can still lead to incorrect classifications by neural networks \cite{ref10,refAP}. In the context of power systems, such perturbations can be deliberately injected by malicious attackers into grid cyber-physical system through false data injection attacks (FDIA) \cite{ref11,ref12}, which can evade conventional anomaly detection in the SCADA system. The difficulties mentioned necessitate the quest for advanced EU estimation methods.

In the realm of DN, AU represents the irreducible stochasticity of the grid, including fluctuations in RES \cite{ref13}, electricity prices or load behaviors, that cannot be reduced through further data collection. 
AU-aware methods model a whole distribution of returns, not just the scalar expectation~\cite{ref20}. 
Effectively capturing such inherent stochasticity is helpful because it enables the system to avoid high-risk actions that might otherwise appear optimal based on expected returns but could lead to detrimental worst-case outcomes.

By juxtaposing the characteristics of AU and EU, it becomes evident that the two types of uncertainty require distinct handling strategies. However, most previous methods concentrate on reducing the probability of constraint violations~\cite{ref14,ref15}, or mainly estimate the combined total uncertainty without differentiating between its sources.  
Existing studies address only part of the uncertainty-aware control problem. 
Constrained and safe DRL methods improve feasibility in DN operation, but mainly under in-distribution stochasticity and without identifying when the learned controller becomes epistemically unreliable \cite{Wang2020VoltVar,ref23}. 
Risk-sensitive and distributional RL methods capture return variability and tail risk, yet they do not explicitly distinguish irreducible environmental randomness from model uncertainty \cite{ref16,ref17}. 
Bayesian DRL and probabilistic OPF approaches introduce predictive uncertainty into decision making, but generally treat uncertainty as an aggregated quantity rather than decomposing it into AU and EU with different operational roles \cite{ref18,ref19}. 
Bayesian UQ has also been used for power-system forecasting and risk assessment, but mostly in offline prediction or assessment settings rather than in closed-loop DRL control with OOD-triggered fallback \cite{Liu2024TIIrisk,Yang2020TIIload}. 
  
The lack of UQ in naive DRL poses a barrier to its deployment in real-world safety-critical DN operations. To close the gaps mentioned, we aim to develop an uncertainty-aware framework to handle OOD cases and manage risks, subsuming previous distributional DRL and Bayesian DRL methods as submodules. The main contributions are summarized as follows:
 
\begin{itemize}
	\item[(1)] 
	We formulate risk and anomaly identification in distribution-network operation as a 
	\emph{return-law uncertainty quantification} problem.
	Previous methods consider only a single type of uncertainty, whereas our approach allows the total uncertainty to be further split into AU and EU. Specifically, we adopt a second-order UQ method to distinctly characterize risks and anomalies.

	\item[(2)] We develop a critic-side return-space instantiation of distance-based second-order UQ, beyond the original finite-label categorical instantiation. By using one-dimensional quantile representations and the Wasserstein barycenter, we convert the original distance-based definitions into tractable forms for critic-side AU and EU. We establish the equivalence of these forms to the original definitions and prove the Monte Carlo consistency of their empirical estimators.

	\item[(3)] We propose an epistemic-reliability gated deployment scheme that
	turns critic-side EU into an online fallback indicator.  
	When the actor's action becomes epistemically unreliable, control is switched from fast DRL inference to a conservative MISOCP-based OPF controller. Thus, EU is not used as a passive anomaly score, but as an
	operational switching signal that couples data-driven control in familiar regimes with optimization-based recovery under distribution shift.
	
\end{itemize}
The rest of this paper is organized as follows. Section~II rigorously formulates the DN operation problem. Section~III proposes the main framework. Section~IV conducts numerical simulations to demonstrate effectiveness of the proposed method. 
Section~V presents the conclusions, and the Appendix provides detailed proofs.

\section{Optimal DN Operation Formulation}
\label{sec:DN_formulation}

We first cast the DN operation problem into a constrained Markov decision process (CMDP), then specify the physical model and operational constraints of the DN.

\subsection{Preliminaries on CMDP}

We consider a discounted constrained Markov decision process (CMDP)
\(
\langle \mathcal{S}, \mathcal{A}, P, R, C, \gamma \rangle
\),
where \(\mathcal{S}\) is the state space, \(\mathcal{A}\) is the action space,
\(P : \mathcal{S} \times \mathcal{A} \times \mathcal{S} \to [0,1]\) is the transition kernel,
\(R : \mathcal{S} \times \mathcal{A} \to \mathbb{R}\) is the reward function,
\(C : \mathcal{S} \times \mathcal{A} \to \mathbb{R}^m\) is a vector of constraint costs,
and \(\gamma \in [0,1)\) is the discount factor.
A stationary policy \(\pi : \mathcal{S} \to \Delta(\mathcal{A})\) maps each state
to a distribution \(\pi(\cdot \mid s)\) over actions
(in practice we often use deterministic policies, written as \(a_t = \pi(s_t)\)).
For an initial state \(s_0\), the discounted return and constraint returns of policy \(\pi\) are
\begin{align}
	J_R^\pi(s_0)
	&= \mathbb{E}^\pi \!\left[
	\sum_{t=0}^{\infty} \gamma^t R(s_t, a_t)
	\right],
	\label{eq:cmdp_return}\\[0.2em]
	J_{C_i}^\pi(s_0)
	&= \mathbb{E}^\pi \!\left[
	\sum_{t=0}^{\infty} \gamma^t C_i(s_t, a_t)
	\right],
	\quad i = 1,\dots,m.
	\label{eq:cmdp_constraint_return}
\end{align}
The CMDP objective is
\begin{align}
	\max_{\pi} \quad & J_R^\pi(s_0),
	\label{eq:cmdp_obj}\\
	\text{s.t.} \quad & J_{C_i}^\pi(s_0) \le d_i,
	\quad i = 1,\dots,m.
	\label{eq:cmdp_constraints}
\end{align} 
We employ Lagrangian relaxation to address the constrained optimization.
Instead of maintaining separate multipliers for each constraint, we aggregate all operational limits into a unified scalar cost \(C(s,a)\) associated with a single multiplier \(\lambda \ge 0\), and formulate the Lagrangian objective as:
$
	L(\pi,\lambda)
	=
	J_R^\pi(s_0)
	-
	\lambda \bigl( J_{C}^\pi(s_0) - d \bigr),
	\label{eq:lagrangian}
$
where \(J_{C}^\pi(s_0)\) denotes the expected discounted cumulative aggregated cost and \(d\) is the tolerance limit.
This yields the saddle point problem:
$
	\max_{\pi}   \min_{\lambda \ge 0}   L(\pi,\lambda),
	\label{eq:cmdp_dual}
$
which we solve by primal--dual updates. 
At iteration $k$, given policy $\pi_k$, we update the Lagrange multiplier (LM) by projected gradient ascent on the dual:
$
	\lambda_{k+1}
	=
	\Pi_{\mathbb R_{+}}
	(
	\lambda_k
	+
	\eta_\lambda\bigl(J_C^{\pi_k}(s_0)-d\bigr)
	),
	\label{eq:lambda_update}
$
where $\eta_\lambda>0$ is a stepsize and $\Pi_{\mathbb R_{+}}(x)=\max\{0,x\}$ enforces $\lambda\ge 0$.
For a fixed \(\lambda\), this is equivalent to solving an unconstrained MDP with the Lagrangian reward:
\begin{equation}
	R_\lambda(s,a)
	=
	R(s,a) - \lambda C(s,a).
	\label{eq:lagrangian_reward}
\end{equation}
The corresponding state--action value function is
\begin{equation}
	Q_\lambda^\pi(s,a)
	=
	\mathbb{E}^\pi \!\left[
	\sum_{t=0}^{\infty} \gamma^t R_\lambda(s_t, a_t)
	\,\Big|\, s_0 = s, a_0 = a
	\right],
	\label{eq:lagrangian_q}
\end{equation}
satisfying the Bellman equation:
\begin{equation}
	Q_\lambda^\pi(s,a)
	=
	R_\lambda(s,a)
	+
	\gamma \,
	\mathbb{E}_{s' \sim P(\cdot \mid s,a)}
	\bigl[ Q_\lambda^\pi(s', \pi(s')) \bigr].
	\label{eq:bellman_lagrangian_q}
\end{equation} 
Instead of the scalar value \(Q_\lambda^\pi(s,a)\), distributional reinforcement learning
models the full return distribution.
Let \(Z^\pi(s,a)\) denote a random variable with the same distribution as the discounted
return under \(R_\lambda\) starting from \((s,a)\).
It satisfies the distributional Bellman equation:
\begin{equation}
	Z^\pi(s,a)
	\;\stackrel{d}{=}\;
	R_\lambda(s,a)
	+
	\gamma \,
	Z^\pi(s', a'),
	\label{eq:dist_bellman}
\end{equation}
where \(\stackrel{d}{=}\) denotes equality in distribution,
\(s' \sim P(\cdot \mid s,a)\) and \(a' \sim \pi(\cdot \mid s')\).

\subsection{Physical model and operational constraints}

We instantiate the CMDP tuple by a radial DN model that determines
the state and action spaces, the transition kernel, the operating
cost, and the operational constraints. The physical model specifies
how the state evolves under control actions, the economic objective
defines the reward, and the engineering limits induce the CMDP
constraints in \eqref{eq:cmdp_return}–\eqref{eq:cmdp_constraints}.

\subsubsection{States and actions}

Consider a radial DN with bus set $\mathcal{N}$, line set $\mathcal{L}$,
distributed generators (DGs) $\mathcal{G}$, energy storage systems (ESSs) $\mathcal{E}^{\mathrm{ESS}}$,
and shunt capacitor banks (SCBs) $\mathcal{C}$, operated in discrete time
$t = 0,1,\dots$ with step length $\Delta t$.
All power quantities are defined positive when injected into the DN,
and reactive power $Q>0$ (resp. $Q<0$) denotes inductive (resp. capacitive) injection.
The state $s_t \in \mathcal{S}$ collects bus voltage magnitudes $V_{i,t}$ and
angles $\theta_{i,t}$ for $i \in \mathcal{N}$, DG outputs, ESS states of charge (SOC),
tap positions of on-load tap changers (OLTCs), SCB statuses, and relevant exogenous
information such as load and renewable forecasts.
The action $a_t \in \mathcal{A}$ aggregates DG active and reactive set-points
$\{P_{i,t}^{\mathrm{DG}}, Q_{i,t}^{\mathrm{DG}}\}_{i \in \mathcal{G}}$,
ESS charging and discharging powers
$\{P_{k,t}^{\mathrm{ch}}, P_{k,t}^{\mathrm{dis}}\}_{k \in \mathcal{E}^{\mathrm{ESS}}}$,
SCB switching commands and OLTC tap positions.

\subsubsection{Operating cost and CMDP reward}

The economic objective is encoded directly in the CMDP stage reward.
For each state–action pair $(s_t,a_t)$, the reward is defined as the
negative total operating cost, including energy purchase from the
upstream grid, local generation costs, and network losses,
\begin{align}
	R(s_t, a_t)
	&=
	-\Big(
	P^{\mathrm{grid}}_t \, \mathcal{P}_t
	+
	\sum_{i \in \mathcal{G}}
	\big(
	a_i (P_{i,t}^{\mathrm{DG}})^2
	+ b_i P_{i,t}^{\mathrm{DG}}
	+ c_i
	\big)
	\nonumber\\
	&\qquad\quad
	+
	\alpha_{\mathrm{loss}}
	\sum_{(i,j) \in \mathcal{L}} P_{ij,t}^{\mathrm{loss}}
	\Big),
	\label{eq:operating_cost_reward}
\end{align}
where $\mathcal{P}_t$ is the electricity price, $a_i, b_i, c_i$ are DG
cost coefficients, $P_{ij,t}^{\mathrm{loss}}$ is the active loss on line
$(i,j)$, and $\alpha_{\mathrm{loss}} > 0$ converts losses into monetary
units. 

\subsubsection{Operational constraints}

Operational security limits and device ratings are modeled as inequality
constraints that define the feasible set of $(s_t,a_t)$.
These inequalities are encoded in the CMDP through the constraint
functions $C_i$ and bounds $d_i$ in
\eqref{eq:cmdp_constraint_return}–\eqref{eq:cmdp_constraints}.
Voltage magnitude limits and line thermal limits are enforced as
$
	V_i^{\min}
	 \le
	V_{i,t}
	\le
	V_i^{\max},
	  i \in \mathcal{N},
	\label{eq:voltage_limits}$
	$  
	|I_{ij,t}| \le
	I_{ij}^{\max},
	  (i,j) \in \mathcal{L},
	\label{eq:current_limits}
$
where $V_i^{\min}, V_i^{\max}$ and $I_{ij}^{\max}$ are voltage and current limits.
Each ESS $k \in \mathcal{E}^{\mathrm{ESS}}$ with rated energy $E_k^{\mathrm{ESS}}$
satisfies the SOC dynamics
$
	\mathrm{SOC}_{k,t+1}
	=
	\mathrm{SOC}_{k,t}
	+
	\frac{\Delta t}{E_k^{\mathrm{ESS}}}
	\left(
	\eta_k^{\mathrm{ch}} P_{k,t}^{\mathrm{ch}}
	-
	\frac{1}{\eta_k^{\mathrm{dis}}} P_{k,t}^{\mathrm{dis}}
	\right),
	\label{eq:soc_dynamics}
$
together with SOC and power bounds
$
	\mathrm{SOC}_k^{\min}
	\le
	\mathrm{SOC}_{k,t}
	\le
	\mathrm{SOC}_k^{\max},
	\label{eq:soc_bounds} $
	$
	0
	\le
	P_{k,t}^{\mathrm{ch}}
	\le
	P_k^{\mathrm{ch},\max},$
    $
	0
	\le
	P_{k,t}^{\mathrm{dis}}
	\le
	P_k^{\mathrm{dis},\max},
	\label{eq:ess_power_bounds}
$
where $\eta_k^{\mathrm{ch}}, \eta_k^{\mathrm{dis}} \in (0,1]$ are charging and
discharging efficiencies.
Device-level constraints for each DG $i \in \mathcal{G}$ are
$
	0
	\le
	P_{i,t}^{\mathrm{DG}}
	\le
	P_{i}^{\mathrm{DG},\max},
	\label{eq:dg_active_bound} $
	$
	\left(P_{i,t}^{\mathrm{DG}}\right)^2
	+
	\left(Q_{i,t}^{\mathrm{DG}}\right)^2
	\le
	\left(S_i^{\mathrm{DG},\max}\right)^2,
	\label{eq:dg_apparent_limit} $
	$
	\cos\phi_i^{\min}
	\le$
	$
	\frac{P_{i,t}^{\mathrm{DG}}}{\sqrt{
			\left(P_{i,t}^{\mathrm{DG}}\right)^2
			+
			\left(Q_{i,t}^{\mathrm{DG}}\right)^2}}
	\le$
	$
	\cos\phi_i^{\max},
	\label{eq:dg_pf}
$
where $P_i^{\mathrm{DG},\max}$, $S_i^{\mathrm{DG},\max}$ and
$\cos\phi_i^{\min}, \cos\phi_i^{\max}$ denote the corresponding ratings.
Each SCB $c \in \mathcal{C}$ has $n_c^{\max}$ identical steps.
Let $n_{c,t} \in \{0,\dots,n_c^{\max}\}$ be the number of energized steps
and $Q_c^{\mathrm{step}}$ the reactive power per step. Then
$
	Q_{c,t}^{\mathrm{SCB}}
	=
	n_{c,t} \, Q_c^{\mathrm{step}},
    $ $
	0 \le n_{c,t} \le n_c^{\max}.
	\label{eq:scb}
$
OLTC tap positions are likewise restricted to discrete sets.
\subsubsection{Network constraints and AC power flow}

Let $G_{ij} + j B_{ij}$ be the $(i,j)$th admittance matrix entry
and $\mathcal{N}_i$ the set of buses connected to bus $i$.
The nodal injections satisfy the AC power flow equations
$
	P_{i,t}^{\mathrm{inj}}
	=$
	$
	V_{i,t}
	\sum_{j \in \mathcal{N}_i}
	V_{j,t}
	$
	$
	\Big[
	G_{ij} \cos(\theta_{i,t} - \theta_{j,t})
	\nonumber  
	+
	B_{ij} \sin(\theta_{i,t} - \theta_{j,t})
	\Big],
	\label{eq:pf_active} $
	$
	Q_{i,t}^{\mathrm{inj}}
	 =
	V_{i,t}
	\sum_{j \in \mathcal{N}_i}
	V_{j,t}
	\Big[
	G_{ij} \sin(\theta_{i,t} - \theta_{j,t})
	\nonumber  
	-
	B_{ij} \cos(\theta_{i,t} - \theta_{j,t})
	\Big].
	\label{eq:pf_reactive}
$
For each non-slack bus $i$,
$
	P_{i,t}^{\mathrm{inj}}
	 =
	P_{i,t}^{\mathrm{DG}}
	+ 
	P_{i,t}^{\mathrm{dis}}
	-
	P_{i,t}^{\mathrm{ch}}
	-
	P_{i,t}^{\mathrm{load}},
	\label{eq:active_balance}$
	$
	Q_{i,t}^{\mathrm{inj}}
	 =
	Q_{i,t}^{\mathrm{DG}}
	+
	Q_{i,t}^{\mathrm{SCB}}
	-
	Q_{i,t}^{\mathrm{load}},
	\label{eq:reactive_balance}
$
while the slack bus injection includes the imported power $P^{\mathrm{grid}}_t$.
For notational compactness, bus-indexed device variables such as
$P_{i,t}^{\mathrm{DG}},Q_{i,t}^{\mathrm{DG}},P_{i,t}^{\mathrm{ch}},P_{i,t}^{\mathrm{dis}},Q_{i,t}^{\mathrm{SCB}}$
are understood as the aggregated injections of all corresponding devices connected to bus $i$ at time $t$,
and are defined to be zero if bus $i$ hosts no such device.

    
\subsubsection{Scalar constraint cost}

The operational limits of DGs, ESSs, SCBs, and OLTCs are handled by feasible
action parameterization, action clipping, and device operating bounds.
Accordingly, the scalar constraint cost \(C(s_t,a_t)\) used in the
primal--dual update is constructed from the residual network-security terms
evaluated after the AC power-flow calculation, including voltage violations,
line or transformer overloading, and power-flow non-convergence. 
Let \(V_{i,t}\) denote the voltage magnitude at bus \(i\), and let
\(\ell_{e,t}\) denote the loading percentage of line or transformer \(e\).
The aggregate voltage violation is defined as
\(\nu_V(s_t,a_t)=\sum_{i\in\mathcal N}
\left([V_i^{\min}-V_{i,t}]_+ + [V_{i,t}-V_i^{\max}]_+\right)\),
and the aggregate thermal-loading violation is defined as
\(\nu_L(s_t,a_t)=\sum_{e\in\mathcal L\cup\mathcal T}[\ell_{e,t}-100]_+\).
Here, \(\nu_V\) is measured in p.u., whereas \(\nu_L\) is measured in
percentage points. 
To combine these residual terms into a scalar constraint cost, normalization
is absorbed into the penalty coefficients. The resulting constraint cost is
\(C(s_t,a_t)=w_V\nu_V(s_t,a_t)+w_L\nu_L(s_t,a_t)+w_{\mathrm{pf}}I_{\mathrm{pf},t}\),
where \(w_V\), \(w_L\), and \(w_{\mathrm{pf}}\) are penalty weights, and
\(I_{\mathrm{pf},t}\in\{0,1\}\) indicates whether the AC power flow fails to
converge.

\section{Uncertainty-Aware DRL Framework}
\label{sec:framework}

\subsection{Framework Overview}

Building on the CMDP formulation in Section~\ref{sec:DN_formulation}, we
construct an uncertainty-aware DRL framework with three components:
a distributional actor–critic backbone, a second-order UQ module, and an uncertainty-aware training and deployment
scheme. The backbone is a deterministic actor–critic architecture with an
implicit quantile network (IQN)-based critic~\cite{ref22} that approximates the return distribution under the
Lagrangian reward \(R_\lambda\), so that both performance and CMDP
constraints are handled in a unified way.

To quantify uncertainty, the critic produces multiple stochastic
approximations of the return distribution for each state–action pair,
using Monte Carlo dropout \cite{ref24} or deep ensembles
\cite{ref25}. 
These predictive distributions are aggregated to obtain total predictive
uncertainty and then processed by a distance-based second-order scheme
to decompose it into aleatoric and epistemic components. The resulting
AU captures irreducible stochasticity from RES, loads, and prices, while
EU reflects model uncertainty due to limited data or distributional
shift, including OOD conditions and cyber-physical attacks.

During training, EU serves as a curiosity signal to encourage exploration of unseen scenarios, while AU quantifies the inherent environmental risk.
During deployment, EU functions as a reliability indicator: if the EU associated
with the actor's action exceeds a threshold, the actor is treated as
epistemically unreliable and control is handed over to a conservative
mathematical programming-based fallback strategy.

The overall procedure is shown in Algorithm~\ref{alg:uq_lagrangian_actor_critic}.

\begin{algorithm}[!t]
	\caption{Lagrangian Distributional Actor--Critic With Second-Order UQ and Fallback}
	\label{alg:uq_lagrangian_actor_critic}
	\begin{algorithmic}[1]
		\STATE \textbf{Input:} replay buffer $\mathcal D$, discount factor $\gamma$, LM stepsize $\eta_\lambda$, tolerance $d$, ensemble size $M$, dropout samples $K$, missed-detection tolerance $\epsilon_{\mathrm{miss}}$.
		\STATE Initialize actor $\pi_\phi$, target actor $\pi_{\bar\phi}$, Bayesian distributional critics $\{Z_{\psi_m}\}_{m=1}^{M}$, target critics $\{Z_{\bar\psi_m}\}_{m=1}^{M}$, and Lagrange multiplier $\lambda\ge0$.
		
		\FOR{each training episode}
		\STATE Observe initial state $s_0$.
		\FOR{$t=0,1,\ldots,T-1$}
		\STATE Construct the candidate action set $\mathcal A_{\mathrm{cand}}(s_t)$.
		\STATE For each $a\in\mathcal A_{\mathrm{cand}}(s_t)$, estimate the critic mean return and $\mathrm{EU}_{\mathrm{crit}}(s_t,a)$ using \eqref{eq:q_critic}--\eqref{eq:eu_critic}.
		\STATE Select the exploratory action $a_t$ according to \eqref{eq:ucb_exploration}.
		\STATE Execute $a_t$, observe $s_{t+1}$, $R(s_t,a_t)$, and $C(s_t,a_t)$.
		\STATE Compute $r_t^\lambda$ by \eqref{eq:lagrangian_reward} and store the transition in $\mathcal D$.
		\STATE Sample a mini-batch $\mathcal B$ from $\mathcal D$.
		\STATE Update the Bayesian distributional critics by minimizing \eqref{eq:bayesian_critic_loss_actorcritic}.
		\STATE Update the actor by maximizing \eqref{eq:actor_obj}.
		\STATE Update the Lagrange multiplier.
		\STATE Update target networks.
		\ENDFOR
		\ENDFOR
		
	\STATE \textbf{Offline threshold calibration:}
	\STATE Construct $\mathcal D_{\mathrm{cal}}\subset\mathcal D_{\mathrm{hold}}$.
	\STATE Form $\mathcal D_{\mathrm{cal}}^{\mathrm{crit}}$ by retaining states with $\Delta C(s)>\eta_C$ or $\Delta R(s)>\eta_R$.
	\STATE Compute $\mathrm{EU}_{\mathrm{crit}}(s,\pi_\phi(s))$ for all $s\in\mathcal D_{\mathrm{cal}}^{\mathrm{crit}}$.
	\STATE Select $\tau_{\mathrm{fb}}$ by the calibration rule in \eqref{eq:safety_threshold}.
		
		\STATE \textbf{Online deployment:}
		\FOR{each deployment state $s_t$}
		\STATE Compute $a_{\mathrm{actor}}=\pi_\phi(s_t)$ and $\mathrm{EU}_{\mathrm{crit}}(s_t,a_{\mathrm{actor}})$.
		\STATE Select the final control action by the fallback rule in \eqref{eq:fallback_rule}.
		\ENDFOR
	\end{algorithmic}
\end{algorithm}

\subsection{Distributional Actor--Critic Backbone}

We adopt a distributional actor--critic architecture as the
decision-making backbone. It operates on the Lagrangian MDP induced by
the CMDP in Section~\ref{sec:DN_formulation}, with scalar constraint cost
\(C(s,a)\) and Lagrangian reward \(R_\lambda(s,a)\) defined in
\eqref{eq:lagrangian_reward}. The policy is parameterized by a
deterministic actor \(\pi_\phi : \mathcal{S} \to \mathcal{A}\), and the
value function is represented by a distributional critic that
approximates the return distribution \(Z^\pi\) in
\eqref{eq:dist_bellman}.
The critic follows the IQN formulation \cite{ref22}. Given a
state--action pair \((s,a)\) and a sampled quantile fraction
\(\tau \in (0,1)\), the critic \(Z_\psi(s,a,\tau)\) outputs an
approximation of the corresponding quantile of the discounted return
under \(R_\lambda\). Let \((s_t,a_t,r_t^\lambda,s_{t+1})\) be a
transition sampled from the replay buffer, where
\(r_t^\lambda = R_\lambda(s_t,a_t)\), and let
\(\tau,\tau' \sim \mathcal{U}(0,1)\) be independent quantile samples.
The distributional Bellman target is
\begin{equation}
	\hat z_t
	=
	r_t^\lambda
	+
	\gamma \,
	Z_{\bar\psi}\bigl(s_{t+1}, \pi_{\bar\phi}(s_{t+1}), \tau'\bigr),
\end{equation}
where \(\bar\psi\) and \(\bar\phi\) denote target-network parameters.
The critic is trained by minimizing the expected quantile Huber loss
\begin{equation}
	\mathcal{L}_{\text{crit}}
	=
	\mathbb{E}
	\Bigl[
	\rho_{\tau}^{\kappa}\!\bigl(
	\hat z_t - Z_\psi(s_t,a_t,\tau)
	\bigr)
	\Bigr],
	\label{eq:iqn_loss}
\end{equation}
where \(\kappa > 0\) is the Huber threshold and
\begin{align}
	\rho_{\tau}^{\kappa}(u)
	&=
	\bigl|\tau - \mathbf{1}_{\{u<0\}}\bigr|
	L_{\kappa}(u),
	\\
	L_{\kappa}(u)
	&=
	\begin{cases}
		\frac{1}{2} u^{2}, & |u| \le \kappa,\\[0.2em]
		\kappa\bigl(|u| - \frac{1}{2}\kappa\bigr), & |u| > \kappa,
	\end{cases}
\end{align}
is the standard smoothed asymmetric quantile regression loss. 
The actor is updated by maximizing the expected Lagrangian return
estimated by the distributional critic,
\begin{equation}
	J_{\text{act}}(\phi)
	=
	\mathbb{E}_{s \sim \mathcal{D}}
	\Bigl[
	\mathbb{E}_{\tau}
	\bigl[
	Z_\psi\bigl(s, \pi_\phi(s), \tau\bigr)
	\bigr]
	\Bigr],
	\label{eq:actor_obj}
\end{equation}
where \(\mathcal{D}\) is the state distribution induced by the replay
buffer. The policy parameters \(\phi\) are updated by gradient ascent on
\(J_{\text{act}}(\phi)\), while the Lagrange multiplier \(\lambda\) is
updated online using stochastic gradient ascent on the dual objective
based on sampled constraint costs \(C(s_t,a_t)\).
Even without the UQ introduced below, this backbone already yields a
standard constrained distributional RL scheme for DN operation under the
CMDP model in Section~\ref{sec:DN_formulation}.

 \subsection{Bayesian Distributional Critic}
 \label{subsec:bayesian_critic}
 
 Starting from the deterministic distributional actor--critic backbone in
 Section~\ref{sec:framework}, we now endow the
 distributional critic with Bayesian structure. We treat the
 critic networks as approximate Bayesian neural networks and obtain
 stochastic predictions via a combination of Monte Carlo dropout
 \cite{ref24} and deep ensembles \cite{ref25}. 
 Monte Carlo dropout perturbs a single network and yields an essentially
 unimodal approximate posterior around one mode, whereas deep ensembles
 maintain multiple independently trained networks and thus approximate a
 multi-modal posterior over value and policy functions. 
 More specifically, we maintain an ensemble of $M$ IQN-based distributional critics
 $\{Z_{\psi_m}\}_{m=1}^{M}$ with parameters $\psi_m$. Each critic
 approximates the quantile function of the Lagrangian return distribution
 under the current policy. Given a state--action pair $(s,a)$,
	 a quantile fraction $\tau \in (0,1)$, and a dropout mask
	 $\xi_{m,k}$, the stochastic critic output is
	\begin{equation}
		\begin{aligned}
			z_{m,k}(s,a,\tau)
			&= Z_{\psi_m}\bigl(s,a,\tau;\xi_{m,k}\bigr), \\
			&\quad m = 1,\dots,M,\; k = 1,\dots,K,
		\end{aligned}
	\end{equation}
	 where dropout masks $\xi_{m,k}$ are sampled i.i.d.\ as in
	 \cite{ref24}. For fixed $(s,a)$ and critic member $m$, the mapping
	 $\tau \mapsto z_{m,k}(s,a,\tau)$ defines a stochastic approximation of
	 the return quantile function under the $k$-th dropout realization.

 Each critic member is trained with the IQN loss
 \eqref{eq:iqn_loss}, using the actor action
 $ a(s_{t+1})$ as input to the
 target network. Specifically, for a transition
$(s_t,a_t,r_t^\lambda,s_{t+1})$ in the replay buffer, the Bellman target
 for critic $m$ under dropout mask $\xi_{m,k}$ is
	 \begin{equation}
	 	\hat z_{t,m,k}
	 	=
	 	r_t^\lambda
	 	+
 	\gamma \,
	 	Z_{\bar\psi_m}\bigl(
	 	s_{t+1}, \bar a(s_{t+1}),
	 	\tau'; \xi_{m,k}
	 	\bigr),
	 \end{equation}
 where $\bar\psi_m$ denotes the target parameters of critic $m$ and
 $\tau' \sim \mathcal{U}(0,1)$. The critic loss for member $m$ on a
 mini-batch $\mathcal{B}$ is
\begin{equation}
	\begin{aligned}
		\mathcal{L}_{\text{crit}}^{(m)}
		&=
		\mathbb{E}_{(s_t,a_t,r_t^\lambda,s_{t+1}) \in \mathcal{B}}
		\mathbb{E}_{\tau,\tau'}
		\mathbb{E}_{k}
		 \Bigl[
		 \rho_{\tau}^{\kappa}\!\bigl(
		 \hat z_{t,m,k}
		 -
		 \\
		 &\qquad
		 Z_{\psi_m}(s_t,a_t,\tau;\xi_{m,k})
		 \bigr)
		 \Bigr],
	\end{aligned}
\end{equation}
 where the expectations over $\tau,\tau'$ are with respect to the
 uniform distribution on $(0,1)$, and the expectation over $k$ averages
 over dropout realizations.
To prevent ensemble collapse and encourage meaningful epistemic spread, we use a repulsive diversity regularizer based on their mean return estimates
\(\mu_m(s,a) = \mathbb{E}_{\tau}[Z_{\psi_m}(s,a,\tau)]\):
\begin{equation}
	\mathcal{L}_{\text{div}}
	=
    \mathbb{E}_{(s,a)\in\mathcal{B}}
	\Biggl[
	\frac{2}{M(M-1)}
	\sum_{m<m'}
	\bigl(
	\mu_m(s,a) - \mu_{m'}(s,a)
	\bigr)^2
	\Biggr],  
	\label{eq:bayesian_critic_loss_div}
\end{equation}
\begin{equation}
	\mathcal{L}_{\text{Bayes}}
	=
	\frac{1}{M}
	 \sum_{m=1}^{M}
	\mathcal{L}_{\text{crit}}^{(m)}
	-
	\beta_{\text{div}} \,
	\mathcal{L}_{\text{div}}.
	\label{eq:bayesian_critic_loss_actorcritic}
\end{equation}

	 The Bayesian critic
	 module therefore produces \(B=MK\) stochastic return quantile functions,
	 which are flattened into the return-law samples \(\{P_b\}_{b=1}^{B}\) used
	 in Section~\ref{sec:uq_return_critic}.

\subsection{Distance-based Second-order UQ}
\label{sec:second_order_uq}

\subsubsection{General distance-based second-order UQ framework}
\label{sec:second_order_uq_general}
 
 We adopt the distance-based second-order UQ framework of \cite{ref27}.
 Let \(Y\) be an outcome space equipped with its Borel \(\sigma\)-algebra.
 A first-order predictive distribution is a probability measure
 \(p \in \mathcal P(Y)\).
 Higher-order uncertainty about the predictive distribution itself is
 modeled by a \emph{second-order} distribution
 \(Q \in \mathcal O(Y) := \mathcal P(\mathcal P(Y))\), namely a distribution
 over \(p \in \mathcal P(Y)\).
 To compare second-order distributions, we first equip the first-order
 space \(\mathcal P(Y)\) with a metric
 \(d_1 : \mathcal P(Y) \times \mathcal P(Y) \to [0,\infty)\).
 Given \(d_1\), we measure distances on the second-order space
 \(\mathcal O(Y)\) by the Wasserstein--\(p\) metric induced by \(d_1\),
 with \(p \ge 1\): for \(Q,Q' \in \mathcal O(Y)\),
 \begin{equation}
 	W_p(Q,Q')
 	=
 	\inf_{\gamma \in \Gamma(Q,Q')}
 	\Biggl(
 	\int_{\mathcal P(Y) \times \mathcal P(Y)}
 	d_1^p(p,\tilde p)
 	\,\mathrm d \gamma(p,\tilde p)
 	\Biggr)^{1/p},
 	\label{eq:w2_second_order}
 \end{equation}  
 where \(\Gamma(Q,Q')\) is the set of all couplings between \(Q\) and \(Q'\).
 Intuitively, \(W_p(Q,Q')\) is the minimal transport cost needed to move
 mass from \(Q\) to \(Q'\) when transporting one unit of mass from
 \(p\) to \(\tilde p\) costs \(d_1(p,\tilde p)\).
 The distance-based construction in \cite{ref27} quantifies total,
 aleatoric, and epistemic uncertainty by measuring the distance of a given
 \(Q \in \mathcal O(Y)\) to three canonical \emph{reference families}.
 Crucially, these families are chosen to represent \emph{least-uncertain}
 states for the corresponding uncertainty type.
 
 \paragraph*{Reference family for total uncertainty}
 Total certainty corresponds to knowing both the outcome and the predictive
 distribution exactly. This is represented by second-order Dirac masses
 at first-order Dirac measures:
 \begin{equation}
 	\mathcal S_{\mathrm{tot}}
 	=
 	\bigl\{
 	\delta_{\delta_y}
 	:
 	y \in Y
 	\bigr\}
 	\subset
 	\mathcal O(Y),
 \end{equation}
 where \(\delta_y \in \mathcal P(Y)\) is the Dirac measure at \(y\), and
 \(\delta_{\delta_y} \in \mathcal O(Y)\) is the Dirac mass at \(\delta_y\).
 Elements of \(\mathcal S_{\mathrm{tot}}\) carry neither first-order nor
 second-order uncertainty.
 
 \paragraph*{Reference family for aleatoric uncertainty (zero-aleatoric states)}
 To isolate aleatoric uncertainty, \cite{ref27} uses as reference the set
 of second-order distributions supported on \emph{deterministic} predictors
 (first-order Dirac measures). Concretely, for any \(m \in \mathcal P(Y)\),
 define
 \begin{equation}
 	\delta_m
 	=
 	\int_Y \delta_{\delta_y}\,\mathrm d m(y),
 \end{equation}
 and set
 \begin{equation}
 	\mathcal S_{\mathrm{al}}
 	=
 	\Bigl\{
 	\delta_m
 	:
 	m \in \mathcal P(Y)
 	\Bigr\}
 	\subset
 	\mathcal O(Y).
 \end{equation}
 Every \(Q' \in \mathcal S_{\mathrm{al}}\) assigns probability one to
 first-order Dirac measures, hence it has \emph{no aleatoric (first-order)}
 uncertainty. Measuring the distance of a general \(Q\) to
 \(\mathcal S_{\mathrm{al}}\) therefore quantifies how far \(Q\) is from
 having deterministic first-order predictors.
 
 \paragraph*{Reference family for epistemic uncertainty (zero-epistemic states)}
 To remove epistemic uncertainty, the reference family is the set of
 second-order Dirac measures on arbitrary first-order distributions:
 \begin{equation}
 	\mathcal S_{\mathrm{ep}}
 	=
 	\bigl\{
 	\delta_p
 	:
 	p \in \mathcal P(Y)
 	\bigr\}
 	\subset
 	\mathcal O(Y).
 \end{equation}
 An element \(\delta_p\) fixes the predictive distribution \(p\) and thus
 eliminates second-order randomness, meaning it has \emph{no epistemic
 	(second-order)} uncertainty, while \(p\) itself may remain diffuse.
 
 Given these reference families, the distance-based indices are defined as
 \begin{align}
 	U_{\mathrm{tot}}(Q)
 	&=
 	\inf_{Q' \in \mathcal S_{\mathrm{tot}}}
 	W_p\bigl(Q,Q'\bigr),
 	\label{eq:utot_def}
 	\\
 	U_{\mathrm{al}}(Q)
 	&=
 	\inf_{Q' \in \mathcal S_{\mathrm{al}}}
 	W_p\bigl(Q,Q'\bigr),
 	\label{eq:ual_def}
 	\\
 	U_{\mathrm{ep}}(Q)
 	&=
 	\inf_{Q' \in \mathcal S_{\mathrm{ep}}}
 	W_p\bigl(Q,Q'\bigr).
 	\label{eq:uep_def}
 \end{align}
 Thus, \(U_{\mathrm{tot}}(Q)\) measures the distance to complete certainty,
 \(U_{\mathrm{al}}(Q)\) measures the distance to the family of
 zero-aleatoric reference states, and \(U_{\mathrm{ep}}(Q)\) measures the
 distance to the family of zero-epistemic reference states.
 For suitable choices of \(d_1\), these indices satisfy the axioms proposed
 in \cite{ref27}, including invariance and a consistent separation of
 first-order (aleatoric) and second-order (epistemic) effects.

\subsubsection{Return-based second-order UQ for the critic}
\label{sec:uq_return_critic}

We now instantiate the distance-based framework of
Section~\ref{sec:second_order_uq_general} on the one-dimensional return
space and use it to quantify second-order uncertainty for the critic. 

\paragraph*{Outcome space and metrics on return laws}

Let \(Y \subset \mathbb R\) denote the return space, and let
\(\mathcal P(Y)\) be the corresponding space of return laws.
We use \(W_2\) for the 2-Wasserstein distance between return laws, induced by
the squared Euclidean cost on \(Y\); the induced second-order construction is
detailed in Appendix~\ref{app:uq_proofs}.
In the one-dimensional setting, the squared first-order 2-Wasserstein
distance on \(\mathcal P(Y)\) admits the quantile representation
\begin{equation}
	\bigl(W_2(p,\tilde p)\bigr)^2
	=
	\int_0^1
	\bigl[
	Q_p(\tau) - Q_{\tilde p}(\tau)
	\bigr]^2
	\,\mathrm d\tau,
	\label{eq:w2_quantile}
\end{equation}
where \(Q_p\) and \(Q_{\tilde p}\) are the quantile functions of
\(p,\tilde p \in \mathcal P(Y)\), see Lemma~\ref{lem:w2_quantile} in Appendix~\ref{app:uq_proofs}. The IQN critic provides direct access
to approximate quantiles, which makes \eqref{eq:w2_quantile} convenient
for numerical evaluation. 
In what follows, we work with the squared first-order Wasserstein cost \(W_2^2\);
 since the mapping
\(r \mapsto r^2\) is monotone on \([0,\infty)\), this corresponds to a
monotone transformation of the indices in \eqref{eq:utot_def}–\eqref{eq:uep_def}
without changing their ordering. Detailed derivations are given in
Appendix~\ref{app:uq_proofs}.
  
\paragraph*{Second-order UQ for the distributional critic}

For each state--action pair \((s,a)\), the Bayesian distributional critic
(with ensembles and Monte Carlo dropout) yields \(B\) stochastic predictive
return laws. For compactness, we write
\(P_b:=p_{\mathrm{crit}}^{(b)}(\cdot|s,a)\), \(b=1,\dots,B\), and
\(P_\star:=p_{\mathrm{crit}}^\star(\cdot|s,a)\), omitting the dependence on
\((s,a)\) when no ambiguity arises. Here, \(B\) is the total number of
stochastic forward passes.
These draws induce the empirical second-order measure
\begin{equation}
	\widehat Q_B(s,a)
	=
	\frac{1}{B}
	\sum_{b=1}^{B}
	\delta_{P_b},
	\label{eq:q_critic}
\end{equation}
which summarizes higher-order uncertainty over return distributions at \((s,a)\).
For interpretation, we may view \(\{P_b\}_{b=1}^B\)
as i.i.d.\ samples from an underlying population law \(Q_{\mathrm{crit}}(s,a)\).
We choose the ground metric on \(\mathcal P(Y)\) as \(d_1 = W_2\).
Under this choice, critic epistemic uncertainty corresponds to the Fr\'echet dispersion
of \(\widehat Q_B(s,a)\) in \((\mathcal P(Y),W_2)\). Let
\begin{equation}
	P_\star
	\in
	\argmin_{p \in \mathcal P(Y)}
	\frac{1}{B}
	\sum_{b=1}^{B}
	\bigl(W_2(P_b,p)\bigr)^2.
	\label{eq:critic_barycenter}
\end{equation}
be a Wasserstein barycenter of the \(B\) return laws. In one dimension it has a closed form
via quantile averaging: if \(F_b^{-1}(\tau\mid s,a)\) is the generalized quantile of
\(P_b\), then
\(F_\star^{-1}(\tau\mid s,a)=\frac{1}{B}\sum_{b=1}^B F_b^{-1}(\tau\mid s,a)\) for all
\(\tau\in(0,1)\); see Lemma~\ref{lem:1d_barycenter_quantile} in Appendix~\ref{app:uq_proofs}.
We then define the critic epistemic and aleatoric indices as
\begin{equation}
	\mathrm{EU}_{\mathrm{crit}}(s,a)
	:=
	\frac{1}{B}
	\sum_{b=1}^{B}
	\bigl(W_2(P_b,P_\star)\bigr)^2,
	\label{eq:eu_critic}
\end{equation}
\begin{equation}
	\mathrm{AU}_{\mathrm{crit}}(s,a)
	:=
	\frac{1}{B}
	\sum_{b=1}^{B}
	\mathrm{Var}_{Y \sim P_b}(Y),
	\label{eq:au_critic}
\end{equation}
where \(\mathrm{EU}_{\mathrm{crit}}\) measures disagreement across stochastic return laws
around their barycenter, and \(\mathrm{AU}_{\mathrm{crit}}\) averages the intrinsic variance
within each return law.
 Appendix~\ref{app:uq_proofs} shows that
\eqref{eq:critic_barycenter}--\eqref{eq:au_critic} match the general distance-based
definitions of Section~\ref{sec:second_order_uq_general} under the chosen metrics, and
Appendix~\ref{app:eu_au_consistency} establishes almost-sure consistency of the Monte Carlo
estimators in \eqref{eq:eu_critic}--\eqref{eq:au_critic}.

 \subsection{Uncertainty-aware Exploration}
 \label{sec:uq_training}
 
We now describe how EU is used during training. 
At each training state \(s_t\), we form a finite feasible candidate action set 
\(\mathcal A_{\mathrm{cand}}(s_t)\) by assigning admissible discrete choices to individual discrete devices and sampling bounded local perturbations around the actor output for continuous controls.
The unperturbed actor action is also included in \(\mathcal A_{\mathrm{cand}}(s_t)\).
 For each candidate action \(a \in \mathcal A_{\mathrm{cand}}(s_t)\), we evaluate both the critic mean return and the critic-side epistemic uncertainty.
 The exploratory behavior action is selected according to the upper confidence rule
 \begin{equation}
 	a_t^{\mathrm{expl}}
 	\in
 	\argmax_{a \in \mathcal A_{\mathrm{cand}}(s_t)}
 	\mathbb E_{\tau}
 	\bigl[
 	Z_{\psi}(s_t,a,\tau)
 	\bigr]
 	+
 	\alpha_{\mathrm{expl}}
 	\mathrm{EU}_{\mathrm{crit}}(s_t,a),
 	\label{eq:ucb_exploration}
 \end{equation}
 where \(\alpha_{\mathrm{expl}}>0\) controls the relative strength of the epistemic bonus.
 To balance exploitation and uncertainty-driven exploration,
 \(\alpha_{\mathrm{expl}}\) is adapted according to the running magnitudes of the two terms in \eqref{eq:ucb_exploration}.
 Specifically, during training we maintain exponential moving estimates of the absolute critic-return term and the raw EU term over sampled candidate actions, and update \(\alpha_{\mathrm{expl}}\) so that the averaged bonus
 \(\alpha_{\mathrm{expl}}\mathrm{EU}_{\mathrm{crit}}(s_t,a)\)
 occupies a prescribed fraction of the critic-return magnitude.
 This target fraction is gradually reduced as training proceeds, encouraging broader exploration in the early stage and more exploitation-oriented behavior after the critic ensemble becomes better calibrated.
 The epistemic term therefore acts as an intrinsic bonus \cite{ref33} that directs exploration toward insufficiently learned state-action regions.

\subsection{Uncertainty-guided fallback control}
\label{sec:fallback}

To safeguard deployment under distributional shift, we equip the agent with an
uncertainty-guided fallback mechanism. The fallback controller is a single-step mixed-integer
second-order cone programming (MISOCP)-based optimal power flow, which solves a
convex relaxation of the one-period DN operation problem with the same network
and device constraints as in Section~\ref{sec:DN_formulation}, e.g.\ based on
the branch-flow model and its convexification~\cite{ref34}. 
The DG constraints are handled by keeping the capability limit as a second-order cone and rewriting the power-factor bound into linear inequalities $-\kappa_i P_{i,t}^{\mathrm{DG}} \le Q_{i,t}^{\mathrm{DG}} \le \kappa_i P_{i,t}^{\mathrm{DG}}$, where $\kappa_i=\tan(\phi_i^{\max})$ (and similarly for the lower bound via $\phi_i^{\min}$ when asymmetric limits are imposed).
Given the current
state $s$, this MISOCP returns a feasible action $a^{\mathrm{MISOCP}}(s)$ that
minimizes the instantaneous operating cost while enforcing the modeled single-step constraints. It serves as a conservative yet reliable baseline when the
RL policy is deemed epistemically unreliable.

\paragraph*{Offline calibration of epistemic thresholds}

We calibrate the epistemic threshold on a separate hold-out set
\(\mathcal D_{\mathrm{hold}}\) that is not used for training the deployed
policy. On this hold-out environment, we fix a reference optimal policy
\(\pi^{\star}\) and use it to define state-wise performance benchmarks.
For any state \(s\) and policy \(\pi\), let \(J_R^{\pi}(s)\) and
\(J_C^{\pi}(s)\) denote the discounted return and aggregated constraint return
starting from \(s\), as in
\eqref{eq:cmdp_return}--\eqref{eq:cmdp_constraint_return}, with \(s_0=s\) and
\(C\) the scalar constraint cost in Section~\ref{sec:DN_formulation}. For the
deployed policy \(\pi\) and the reference policy \(\pi^{\star}\), we define the
long-horizon constraint and reward gaps as
\begin{align}
	\Delta C(s)
	&=
	\max\!\left(0,\, J_C^{\pi}(s) - J_C^{\star}(s)\right),
	\\
	\Delta R(s)
	&=
	\max\!\left(0,\, J_R^{\star}(s) - J_R^{\pi}(s)\right).
\end{align}
Thus, \(\Delta C(s)>0\) means that the deployed policy incurs higher cumulative
constraint cost than the reference policy, while \(\Delta R(s)>0\) means that
it obtains lower cumulative reward. Given \(\eta_C,\eta_R\ge 0\), let
\(\mathcal D_{\mathrm{cal}}^{\mathrm{crit}}\) collect the hold-out calibration
states with \(\Delta C(s)>\eta_C\) or \(\Delta R(s)>\eta_R\). This set is
separate from the training data and the final evaluation episodes. For each
\(s\in\mathcal D_{\mathrm{cal}}^{\mathrm{crit}}\), we compute
\(e_s=\mathrm{EU}_{\mathrm{crit}}(s,\pi_\phi(s))\). The fallback threshold is
chosen as the empirical \(\epsilon_{\mathrm{miss}}\)-lower quantile:
\begin{equation}
	\tau_{\mathrm{fb}}
	=
	\widehat Q_{\epsilon_{\mathrm{miss}}}
	\left(
	\{e_s:s\in\mathcal D_{\mathrm{cal}}^{\mathrm{crit}}\}
	\right).
	\label{eq:safety_threshold}
\end{equation}
This rule directly controls the fraction of critical calibration states that
are covered by fallback.

\paragraph*{Online decision rule}

At deployment time, given the current state \(s_t\), the agent first computes
the actor action \(a_{\mathrm{actor}}=\pi_\phi(s_t)\) and the corresponding
epistemic score \(\mathrm{EU}_{\mathrm{crit}}(s_t,a_{\mathrm{actor}})\). The
final control action is selected by
\begin{equation}
	a_t
	=
	\begin{cases}
		a^{\mathrm{MISOCP}}(s_t),
		&
		\mathrm{EU}_{\mathrm{crit}}(s_t,a_{\mathrm{actor}})
		\ge \tau_{\mathrm{fb}},
		\\[0.3em]
		a_{\mathrm{actor}},
		&
		\mathrm{EU}_{\mathrm{crit}}(s_t,a_{\mathrm{actor}})
		< \tau_{\mathrm{fb}}.
	\end{cases}
	\label{eq:fallback_rule}
\end{equation}
States with \(\mathrm{EU}_{\mathrm{crit}}(s_t,a_{\mathrm{actor}})
\ge \tau_{\mathrm{fb}}\) are treated as epistemically unreliable and handed
over to the MISOCP-based fallback controller.


\section{Simulation Results}
\label{sec:sim}
 
\subsection{Experimental Protocol and OOD Scenario Construction}
\label{subsec:exp_protocol}

\subsubsection{Test Cases and Network Adaptation}
\label{subsubsec:test_cases}

We organize the numerical study around two test cases. 
\textbf{Test Case I} is based on the MV Oberrhein network in pandapower~\cite{ref35}, a
realistic 20~kV medium-voltage distribution network supplied by two
substations. The original static profiles are converted into a 24-hour
operation task by using UCI load curves~\cite{ref37} while preserving the
original load power factors, and Kaggle PV generation curves~\cite{ref36} for
renewable injections. The electricity price during 8{:}00--21{:}00 is set to
1.558 times the off-peak price. We add 10 SCBs, each with four
0.12~MVAR steps, and 10 ESSs, each with 2~MWh capacity and 0.5~MW maximum
charging/discharging power. Among the 153 static generators in the original
network, one third are treated as PV units and the rest as diesel DGs with a
minimum power factor of 0.7. Bus voltages are constrained within
0.95--1.05~p.u., and the two OLTC transformers are modeled with tap positions
from \(-9\) to \(9\) and a 1.5\% tap step.  
\textbf{Test Case II} is adapted from the IEEE European low-voltage network in
pandapower~\cite{ref35}.
The original benchmark is a 0.416~kV radial feeder with 907 buses and 905 lines. 
We convert this benchmark into a balanced three-phase
operation case by assigning equal per-phase load and generation profiles at
each active bus.
We extend the transformer
to a controllable tap-changing transformer.
We add 51 PV units. 
The action space includes 102 dispatchable DGs, 13 ESSs, 10 SCBs, and 1 OLTC
transformer. These devices form
a 286-dimensional action space.

To check whether the constraint-compliance behavior of the proposed method is
sensitive to the scalarization of network-security violations, we conduct a
penalty-ratio ablation in Test Case II. The default setting uses
\(w_V:w_L=1:1\), with \(w_V=w_L=10^3\), while power-flow non-convergence is
treated as a hard failure and assigned a much larger penalty,
\(w_{\mathrm{pf}}=10^6\). For the ablation, the voltage/loading penalties are
kept at the same overall scale and only their relative ratio is varied, namely
\(w_V:w_L\in\{0.75:1,1:1,2:1\}\).

\begin{table}[!ht]
	\centering
	\caption{Sensitivity to the relative voltage/loading penalty weights.}
	\label{tab:penalty_ratio_sensitivity}
	\footnotesize
	\setlength{\tabcolsep}{3.5pt}
	\begin{tabular}{lccc}
		\toprule
		\(w_V:w_L\) & Reward & Voltage viol. & Loading viol. \\
		\midrule
		\(0.75:1\)
		& \(-17.24 \pm 0.39\)
		& \((7.1 \pm 3.8)\!\times\!10^{-3}\)
		& \(3.2 \pm 2.8\) \\
		\(1:1\)
		& \(-17.58 \pm 0.33\)
		& \((4.6 \pm 2.7)\!\times\!10^{-3}\)
		& \(5.6 \pm 4.2\) \\
		\(2:1\)
		& \(-17.82 \pm 0.27\)
		& \((2.3 \pm 1.5)\!\times\!10^{-3}\)
		& \(9.1 \pm 7.6\) \\
		\bottomrule
	\end{tabular}
\end{table}

Table~\ref{tab:penalty_ratio_sensitivity} shows the expected tradeoff:
larger \(w_V\) reduces voltage violations, whereas relatively larger loading
penalties reduce thermal violations and yield a slightly higher reward. Across
all ratios, both violation terms remain small, indicating that the safety
behavior of the proposed method is not tied to a narrowly tuned penalty ratio.

\subsubsection{OOD Scenario Construction}
\label{subsubsec:ood_scenario_construction}
 
The OOD scenarios include four complementary families. First, observation-level
OOD cases perturb the agent input while keeping the physical exogenous
trajectories unchanged. This family includes additive Gaussian observation noise
and FDIA-style adversarial sensor manipulation. For Gaussian observation noise,
the deployed policy observes \(\tilde o_t=o_t+\epsilon_t\), where
\(\epsilon_t\sim\mathcal N(0,\sigma^2 I)\) and
\(\sigma\in\{0.5,1.0\}\). For FDIA-style attacks, only
measurement-related state components are perturbed, including bus voltages,
device telemetry, branch loading, and grid exchange. The attacked observation is written as
\(\tilde o_t=o_t+m\odot\delta_t\), where \(m\) is a binary mask selecting the
measurement-related components and \(\|\delta_t\|_{\infty}\le\varepsilon\).
We consider both white-box and black-box attack variants. In the white-box case,
the attacker has access to the learned model and uses projected gradient steps
against the critic objective, with
\(\delta_t^{k+1}=\Pi_{\|\delta\|_{\infty}\le\varepsilon}
(\delta_t^k+\eta_{\mathrm{att}}\operatorname{sign}
(\nabla_{o_t}J_{\mathrm{att}}(o_t+m\odot\delta_t^k)))\), where
\(J_{\mathrm{att}}\) is chosen to decrease the critic-estimated return,
increase the critic-estimated cost, or combine the two objectives. In the
black-box case, the attacker has no access to the model parameters and only
queries the deployed agent. The scalar query objective is defined as
\(f(\tilde o_t):=-\mathbb E_{\tau}
[Z_{\psi}(\tilde o_t,\pi(\tilde o_t),\tau)]\), and the attack direction is
estimated by the symmetric finite-difference estimator
\(\hat g_t=\frac{1}{K_q}\sum_{k=1}^{K_q}
\frac{f(\tilde o_t+\zeta u_k)-f(\tilde o_t-\zeta u_k)}
{2\zeta}u_k\), where \(u_k\sim\mathcal N(0,I)\), \(K_q\) is the number of
query probes, and \(\zeta>0\) is the probing scale. The black-box perturbation
is then updated by
\(\delta_t^{k+1}=\Pi_{\|\delta\|_{\infty}\le\varepsilon}
(\delta_t^k+\eta_{\mathrm{att}}\operatorname{sign}(m\odot\hat g_t))\).
The default FDIA budget is \(\varepsilon=0.1\), with projected steps of size
\(0.025\).
Second, synthetic profile-shift scenarios impose structured deviations on the
daily exogenous trajectories, including load surges, PV dropouts, price spikes,
low-renewable high-load periods, evening net-load ramps, renewable
overgeneration, and cloud-ramp PV attenuation. These cases emulate operationally
meaningful stress events with randomized affected windows and stress magnitudes.
Third, frequency-domain OOD cases perturb the spectral components of the load
and PV trajectories and then reconstruct the time-domain profiles. This changes
the temporal shape of the exogenous curves while keeping their daily scale
comparable to the original profiles. 
Fourth, data-driven OOD cases are built from hold-out or external real-world datasets that are not used during training, including unseen UCI load profiles~\cite{ref37}, London smart-meter household consumption data~\cite{refLCL_IEEE}, hold-out Kaggle PV curves~\cite{ref36}, and NSRDB irradiance and temperature records~\cite{refNSRDB_IEEE_URL}. 
For external load scenarios, the London smart-meter readings are first aggregated to hourly resolution to match the 24-step scheduling horizon. The aggregated profile is converted into a unit-mean temporal multiplier \(r_t^{\mathrm{ext}}=L_t^{\mathrm{ext}}/(\frac{1}{T}\sum_{\tau=1}^{T}L_{\tau}^{\mathrm{ext}})\), and then applied to the original bus-level nominal active loads as \(P_{i,t}^{\mathrm{load,ext}}=\alpha_d P_i^{\mathrm{load},0} r_t^{\mathrm{ext}}\), where \(P_i^{\mathrm{load},0}\) is the nominal active load at bus \(i\), and \(\alpha_d\) rescales the daily energy to the same operating range as the in-distribution cases. The reactive load is reconstructed using the original bus power factor, \(Q_{i,t}^{\mathrm{load,ext}}=P_{i,t}^{\mathrm{load,ext}}\tan(\arccos(\mathrm{pf}_i^0))\). 
Therefore, the feeder-specific spatial allocation and power-factor structure are inherited from the target network, while the external data only alter the temporal demand shape.
For NSRDB-based PV scenarios, the irradiance and temperature records are used to construct a bounded PV availability factor rather than a detailed inverter model. Specifically, \(\rho_t^{\mathrm{NSRDB}}=\min\{1,\max\{0,(\mathrm{GHI}_t/1000)[1+\beta_{\mathrm{pv}}(T_t-25)]\}\}\), where \(\mathrm{GHI}_t\) is the hourly global horizontal irradiance, \(T_t\) is the ambient temperature, and \(\beta_{\mathrm{pv}}=-0.004~{}^{\circ}\mathrm{C}^{-1}\). The PV active power at bus \(b\) is then \(P_{b,t}^{\mathrm{PV,ext}}=P_{b,\mathrm{rated}}^{\mathrm{PV}}\rho_t^{\mathrm{NSRDB}}\). This construction preserves the installed PV capacities and the network-side operational constraints of the target feeder, while introducing real meteorological temporal patterns. 
Finally, we include VAE-generated near-OOD
trajectories~\cite{ref38}, which are decoded from a latent model trained on
in-distribution profiles and therefore represent milder distribution shifts.

\subsection{Test Case I: Qualitative Demonstration on the Original System}
\label{subsec:testcase1}
For Test Case I, the proposed agent uses \(M=5\) Bayesian distributional
critics, \(K=4\) Monte Carlo dropout samples, and dropout rate \(\rho=0.05\). Training
uses \(\gamma=0.95\), learning rate \(10^{-4}\), buffer size \(2000\),
warm-up length \(1000\), \(\alpha_{\mathrm{expl}}=0.2\), and Lagrange
multiplier stepsize \(10^{-3}\).

\subsubsection{AU Evaluation}
Fig.~\ref{fig:au_testcase1} shows the AU estimates in Test Case I under
different environmental noise scales. The estimated AU generally increases as
the noise scale grows. 

\begin{figure}[!htbp]
	\centering
	\includegraphics[width=0.95\columnwidth]{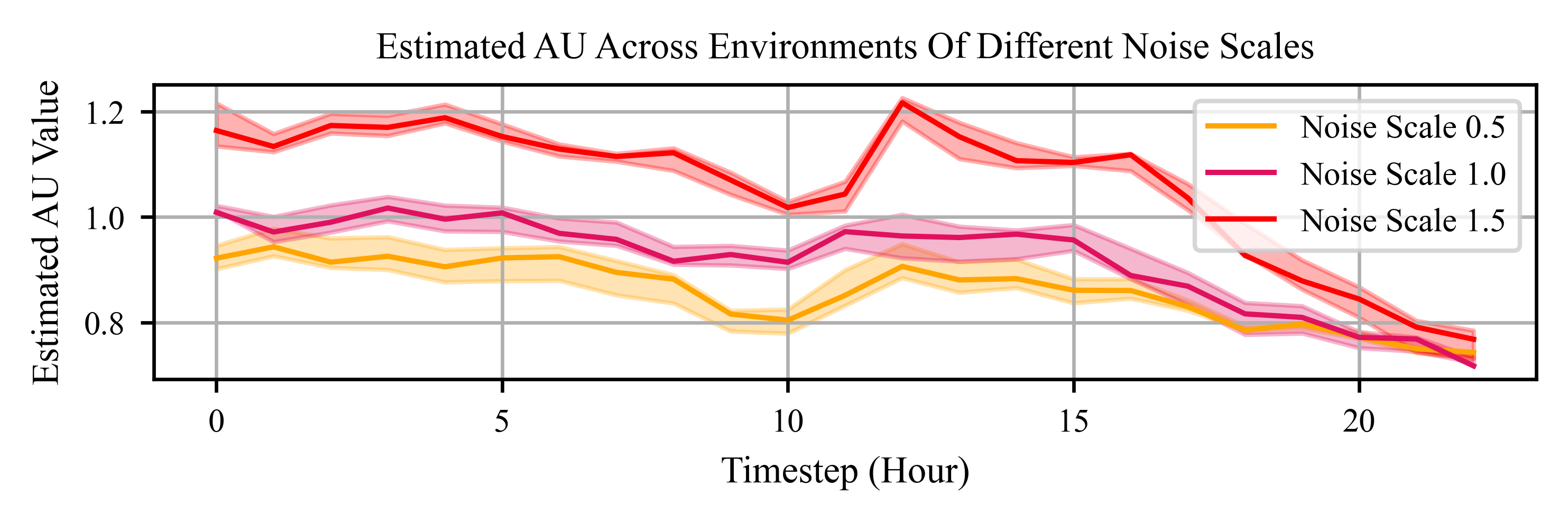}
	\caption{Estimated AU under different environmental noise scales in Test
		Case I. Solid curves show the mean and shaded areas denote one standard
		deviation. This confirms that the
		proposed AU estimate is sensitive to irreducible environmental stochasticity.}
	\label{fig:au_testcase1}
\end{figure}

\subsubsection{EU Evaluation}

Fig.~\ref{fig:eu_testcase1} reports the estimated EU in Test Case I under
different OOD settings. Compared with the InD baseline, all OOD cases produce
clearly higher EU values. 
\begin{figure}[!htbp]
	\centering
	\includegraphics[width= \columnwidth]{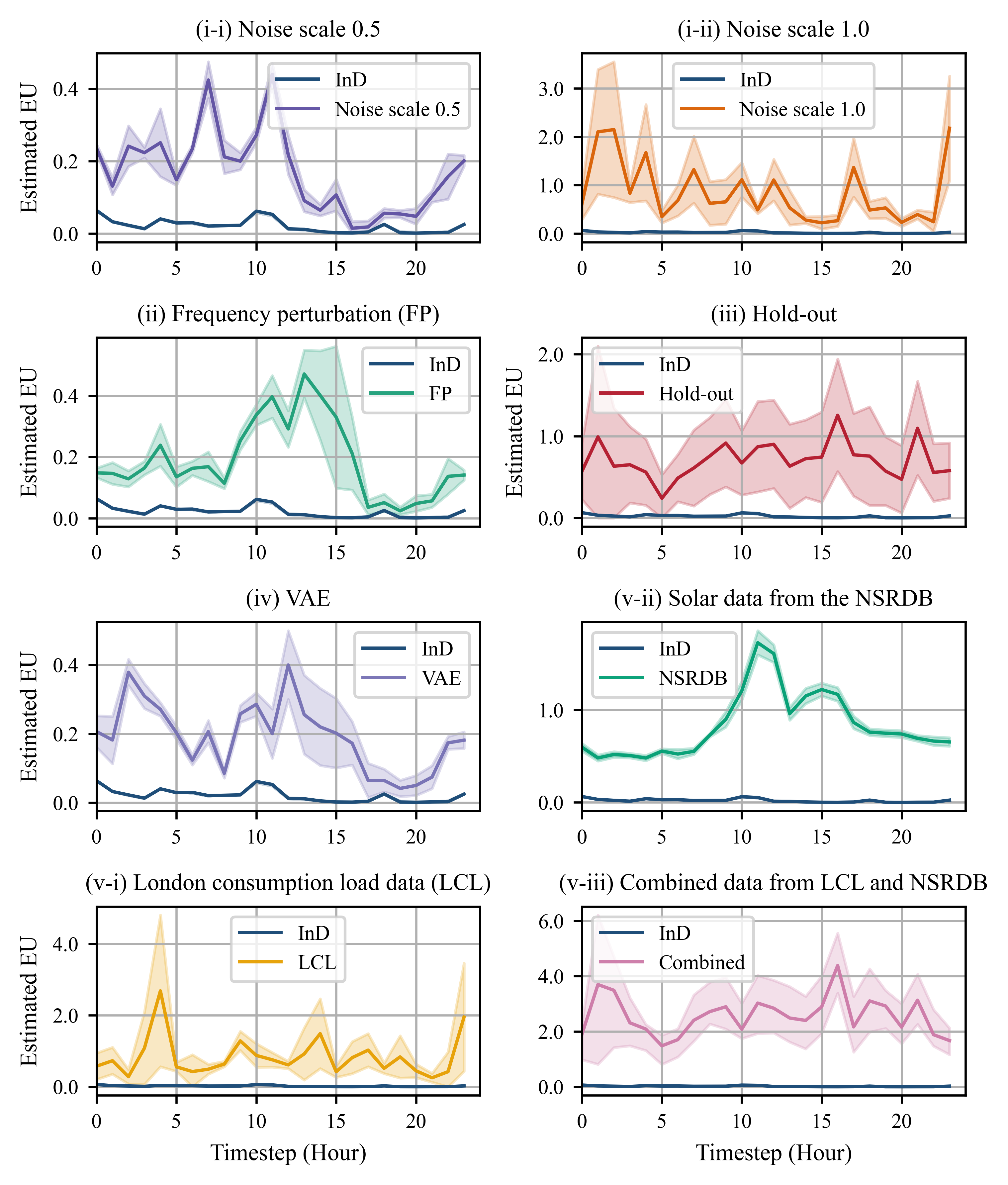}
	\caption{Estimated EU under different OOD scenarios in Test Case I. The InD
		curve is shown as the reference baseline. Solid curves show the mean and
		shaded areas denote one standard deviation. 
		The increase is mild for near-OOD VAE samples and
		frequency perturbations, but becomes much stronger under larger observation
		noise, real-world LCL load data, NSRDB solar data, and their combined shift.
		This shows that the proposed EU estimate can distinguish unfamiliar
		operating conditions from normal InD operation. }
	\label{fig:eu_testcase1}
\end{figure}
Fig.~\ref{fig:eu_attack_testcase1} reports the EU response to FDIA-style sensor
manipulation in Test Case I. As the attack steps increase, EU generally rises
across the 24-hour horizon, especially under the white-box setting. Black-box
attacks also produce increasing EU in several timesteps, although with a
different temporal pattern.  
\begin{figure}[bt]
	\centering
	\includegraphics[width=0.95\columnwidth]{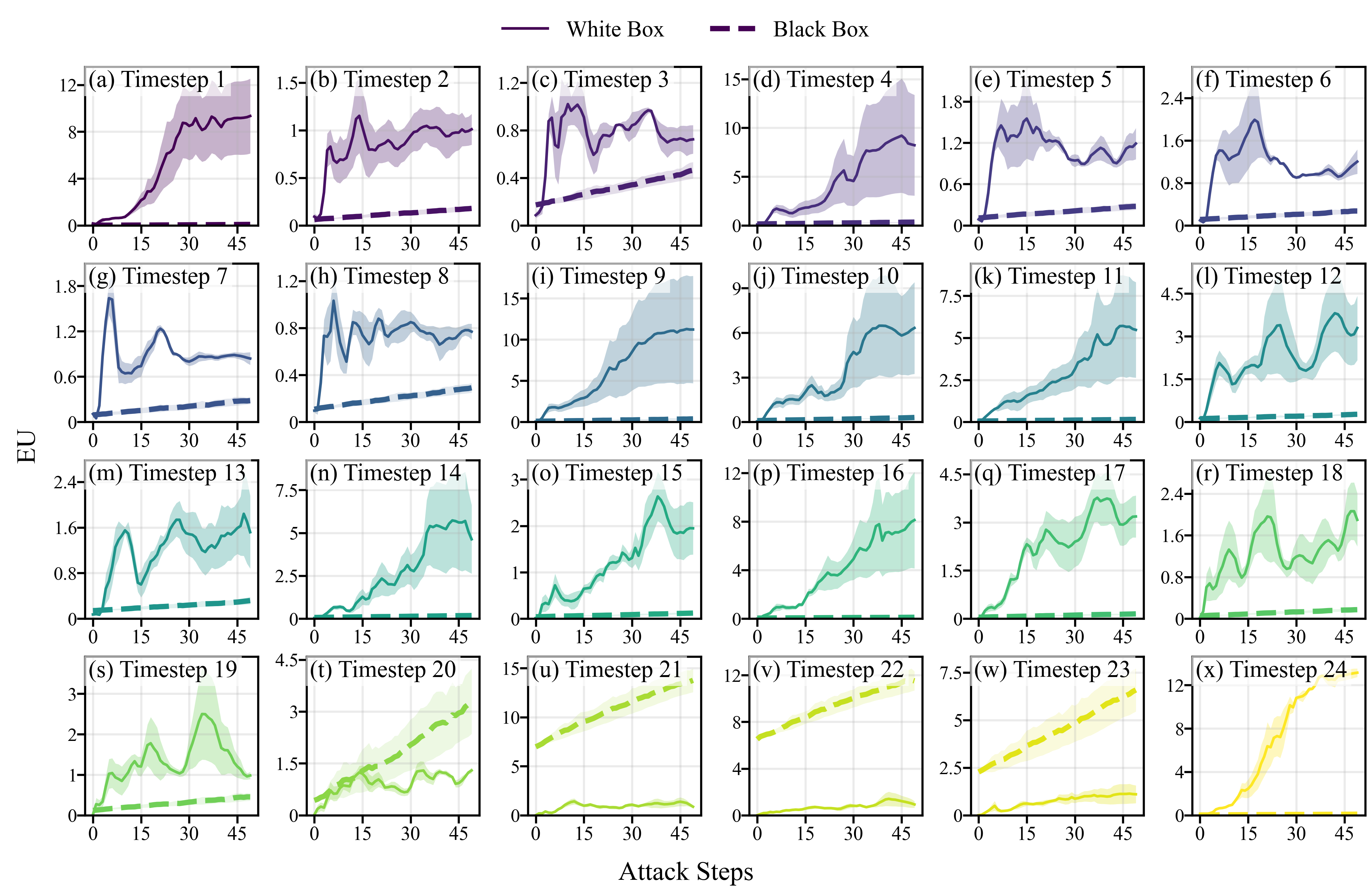}
	\caption{EU response to FDIA-style sensor attacks in Test Case I. This shows that critic-side EU can serve as a
		sensitivity indicator for adversarial observation shifts. 
		Solid
		curves denote white-box attacks, dashed curves denote black-box attacks, and
		shaded areas denote one standard deviation.}
	\label{fig:eu_attack_testcase1}
\end{figure}

\subsection{Test Case II: Quantitative Benchmark Study}
\label{subsec:testcase2}
 
\subsubsection{AU Evaluation}
To evaluate AU estimation in Test Case II, we use the same set of trained
critic parameters and change only the AU estimation method. Specifically, the
evaluation uses five independently trained models, each with six ensemble
members. The proposed AU estimator is compared with an IQN-style member
estimator, where each ensemble member is treated as an independent IQN critic
and its return variance is used as AU. 
As shown in Fig.~\ref{fig:au2}, both estimators increase monotonically with
the environmental noise scale.   
\begin{figure}[!htbp]
	\centering
	\includegraphics[width=\columnwidth]{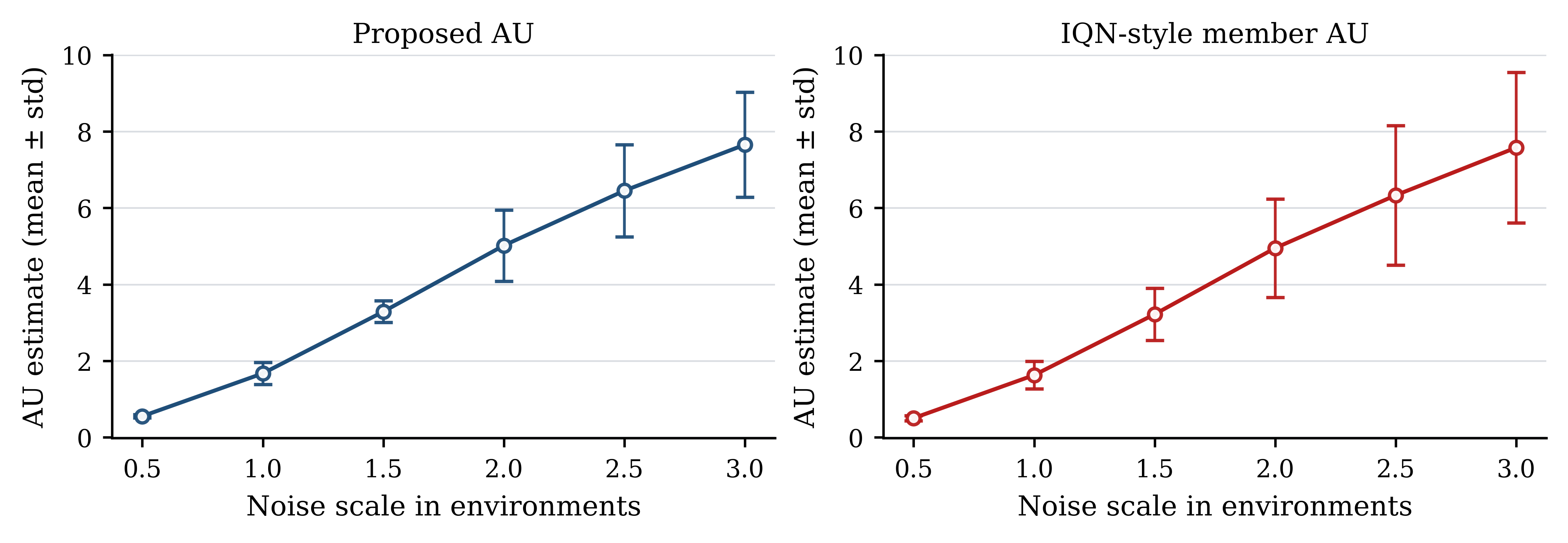}
	\caption{AU evaluation under different environmental noise
		scales. 
		Both methods successfully track the increase of intrinsic stochasticity.  }
	\label{fig:au2}
\end{figure}  
This confirms that
the proposed AU can track the increase of irreducible environmental
stochasticity.
Fig.~\ref{fig:au3} further evaluates the estimator variance of AU,
rather than the AU magnitude itself. 
\begin{figure}[!htbp]
	\centering
	\includegraphics[width=0.8\columnwidth]{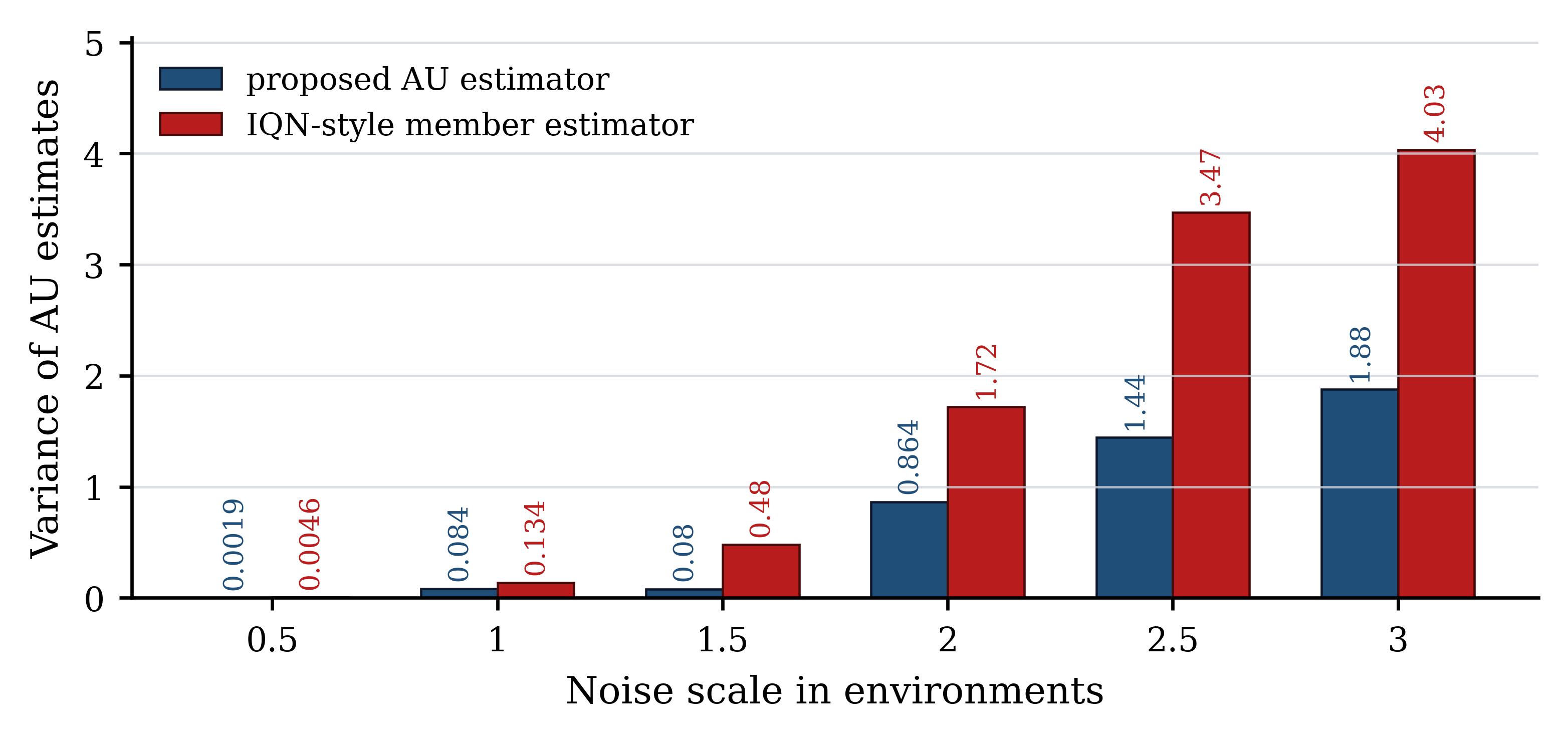}
	\caption{Estimator variance of AU under different
		environmental noise scales. The variance is computed over episode-level AU
		estimates under the same noise scale. The proposed AU estimator shows lower
		estimator variance than the IQN-style baseline.  }
	\label{fig:au3}
\end{figure} 
For each environmental noise scale
\(\sigma_{\mathrm{au}}\), we first compute one episode-level AU estimate for
each rollout, denoted by
\(\widehat{\mathrm{AU}}_e(\sigma_{\mathrm{au}})\). We then report
\(\operatorname{Var}_{e:\sigma_e=\sigma_{\mathrm{au}}}
[\widehat{\mathrm{AU}}_e(\sigma_{\mathrm{au}})]\), where the variance is taken
over all episode-level AU estimates collected under the same environmental
noise scale. Therefore, a smaller value in Fig.~\ref{fig:au3}
indicates lower fluctuation of the AU estimator across repeated rollouts, not
smaller AU of the environment. 
The proposed AU estimator consistently
shows a lower estimator variance than the IQN-style baseline. 
The advantage comes from coupling ensemble prediction with the second-order AU
definition. The ensemble members are used as Monte Carlo samples of the second-order
predictive distribution. The proposed AU then suppresses
member-specific fluctuation while retaining sensitivity to environmental
stochasticity. In contrast, the IQN-style baseline directly uses each member's
raw return variance as an AU estimate, making it more sensitive to calibration
differences, quantile-spread fluctuations, and training randomness.

\subsubsection{EU Evaluation via Training Performance}
\label{subsubsec:eu_training_performance}

We first evaluate whether the proposed EU estimate is useful during training.
Fig.~\ref{fig:eu_training} reports the evaluation reward, raw constraint cost,
and log-scale constraint cost in Test Case II.
The algorithm comparison includes TD3~\cite{refTD3}, PPO~\cite{refPPO},
PPO-GAE~\cite{refPPO,refGAE}, CPO-BDQN~\cite{refCPO,refBootDQN},
Thompson-sampling DQN~\cite{refThompsonDQN}, and IQN~\cite{ref22}.
Among these learning-based baselines, the proposed method achieves the best
final reward and the lowest long-term constraint cost. This indicates that the
performance gain is not only from distributional value learning, but also from
using critic-side EU as an exploration signal.
For the hyperparameter ablations, we adopt a two-stage protocol. First, a full
grid search is conducted in a reduced action space. 
This stage is used to select the default setting
\(\alpha_{\mathrm{expl}}=0.3\), ensemble size \(M=6\), and dropout rate
\(\rho=0.08\). After fixing this setting, we return to the full action space and
conduct single-factor sensitivity tests by varying one hyperparameter at a time.
The exploration-bonus ablation shows that too small
\(\alpha_{\mathrm{expl}}\) gives insufficient epistemic exploration, whereas too
large \(\alpha_{\mathrm{expl}}\) overemphasizes uncertain actions and slows
exploitation. The intermediate value \(\alpha_{\mathrm{expl}}=0.3\) gives the
best balance between reward improvement and cost reduction. The ensemble-size
and dropout-rate ablations show similar robustness. Moderate changes around
\(M=6\) and \(\rho=0.08\) lead to comparable convergence trends, while the
selected setting gives a favorable tradeoff between final reward, learning
stability, and constraint-cost suppression. 

\begin{figure*}[!tbp]
	\centering
	\includegraphics[width=\textwidth]{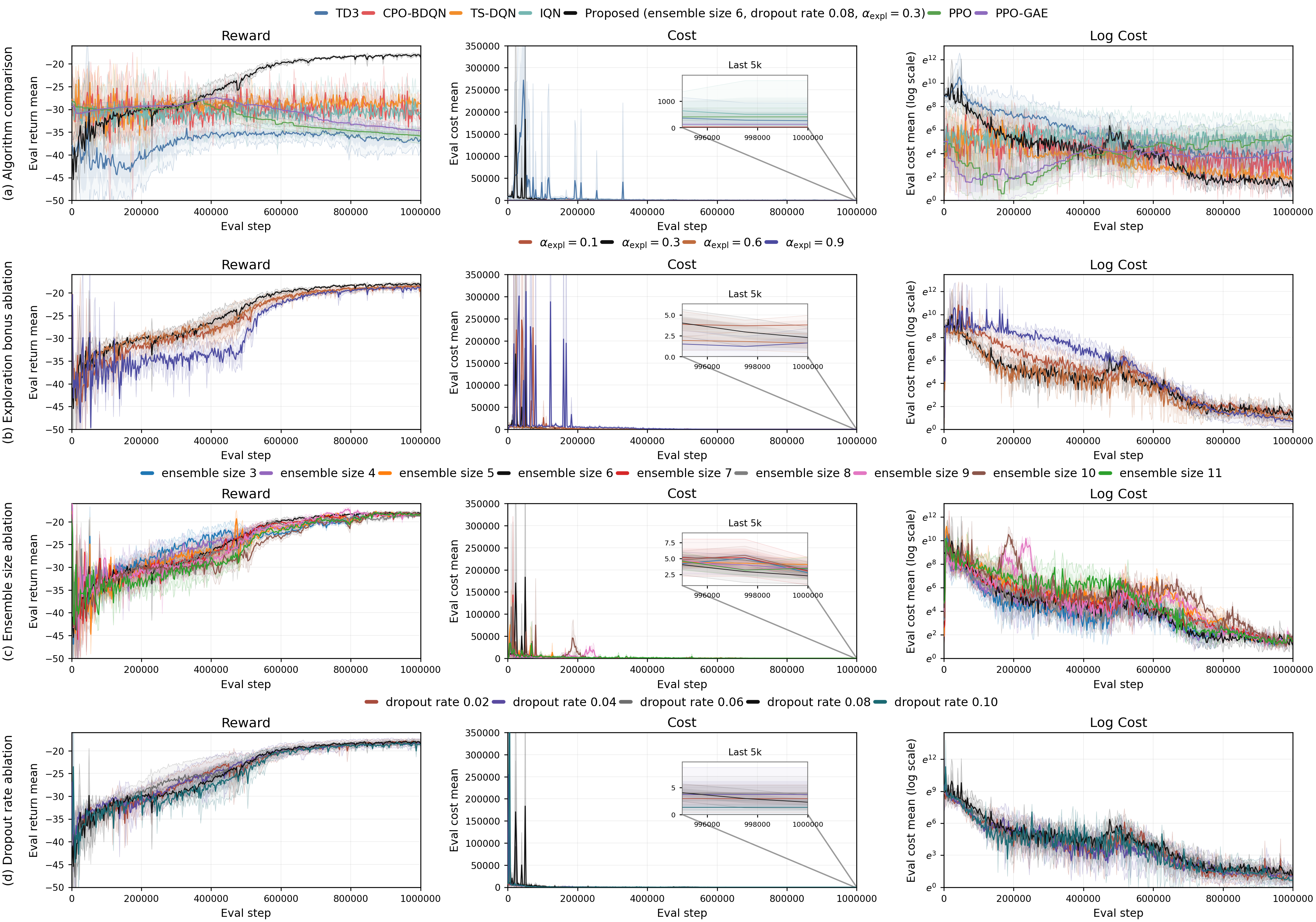}
		\caption{EU evaluation via training performance in Test Case II.
			Each row reports the evaluation reward, raw constraint cost, and log-scale
			constraint cost. Row (a) compares the proposed method with representative
			baselines. Rows (b)--(d) report full action-space sensitivity tests on the
			exploration bonus coefficient, ensemble size, and dropout rate, respectively.
			The offline MPC reference in Test Case II remains fully feasible, with zero
			constraint cost across the evaluated episodes, and yields an average reward
			of \(-16.53\) with low variance \((3.03\times10^{-2})\).}
	\label{fig:eu_training}
\end{figure*}

To examine EU exploration itself, we map each visited state during training to a
coarse bin. The number of unique bins measures
the coverage of distinct reachable operating regions. 
A bin is counted as useful
only when the associated transition is feasible, with converged AC power flow and without voltage or loading violation. The useful ratio is then the
fraction of useful unique bins among all unique bins.
As shown in Fig.~\ref{fig:eu_exploration_bins}, the proposed UQ exploration
achieves both the largest number of unique bins and the highest useful ratio. Compared with fixed Gaussian exploration noise, the proposed strategy produces
broader coverage with a higher useful-bin ratio.
\begin{figure}[htbp]
	\centering
	\includegraphics[width=\columnwidth]{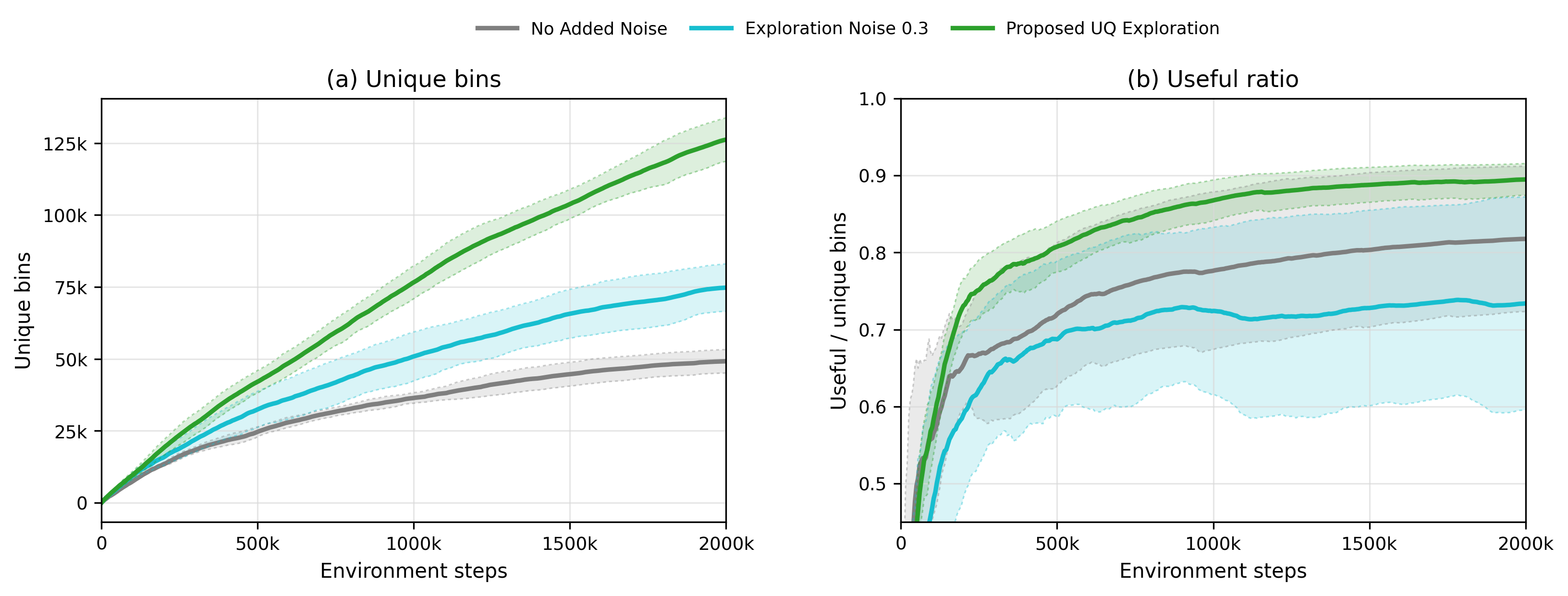}
	\caption{State-space coverage and useful-region ratio under different
		exploration strategies in Test Case II. Panel (a) reports the cumulative
		number of unique bins visited during training. Panel (b)
		reports the ratio of useful unique bins to all unique bins. Solid curves
		show the mean over repeated runs, and shaded areas denote one standard
		deviation. The proposed UQ exploration expands the explored operating
		region while maintaining a higher proportion of feasible and useful states.}
	\label{fig:eu_exploration_bins}
\end{figure}

\subsubsection{EU Evaluation via OOD Detection and Degradation Correlation}

We next evaluate whether the proposed EU can detect OOD states and whether its magnitude is correlated with the resulting operational
degradation. 
 Table~\ref{tab:ood_detection_auroc_fpr95} compares the proposed EU score with
 classical OOD detection baselines. For the proposed method, the
 episode-level EU is directly used as the detection score. For the baselines,
 we train conventional feature-space detectors using actor features extracted
 from the raw observation. The compared baselines include KDE, KNN (\(k=5\)), GMM (\(4\) components), and
 Mahalanobis-distance detectors~\cite{ruff2021unifyingad}. We report both
 AUROC and FPR95, where larger AUROC and smaller FPR95 indicate better OOD
 separation.
 As the ensemble size decreases, the detection performance degrades mildly but
 remains competitive with the feature-space baselines. 
 \begin{table}[h]
	\centering
		\caption{OOD detection performance in Test Case II.}
	\label{tab:ood_detection_auroc_fpr95}
	\begin{tabular}{llcc}
		\hline
		Group & Method & AUROC & FPR95 \\
		\hline
		Ours & Ensemble size 11 & $0.9856 \pm 0.0031$ & $0.0580 \pm 0.0105$ \\
		Ours & Ensemble size 9  & $0.9808 \pm 0.0086$ & $0.0586 \pm 0.0161$ \\
		Ours & Ensemble size 7  & $0.9709 \pm 0.0159$ & $0.0607 \pm 0.0277$ \\
		Ours & Ensemble size 5  & $0.9785 \pm 0.0066$ & $0.0783 \pm 0.0149$ \\
		\hline
		Baseline & KDE    & $0.9770 \pm 0.0030$ & $0.0594 \pm 0.0121$ \\
		Baseline & KNN    & $0.9761 \pm 0.0027$ & $0.0595 \pm 0.0119$ \\
		Baseline & GMM    & $0.9731 \pm 0.0027$ & $0.0641 \pm 0.0080$ \\
		Baseline & Mahal. & $0.9670 \pm 0.0058$ & $0.0842 \pm 0.0101$ \\
		\hline
	\end{tabular}
\end{table}
We further examine whether the EU score is aligned with the severity of
operational degradation. We compare EU with the episode-level
reward gap \(\Delta R\) and constraint-cost gap \(\Delta C\) defined in
Section~\ref{sec:fallback}. As shown in Fig.~\ref{fig:eu_ood_detection}, each
point corresponds to one held-out episode. Larger EU is generally associated
with larger \(\Delta R\) and \(\Delta C\).
\begin{figure}[h]
	\centering
	\includegraphics[width=\columnwidth]{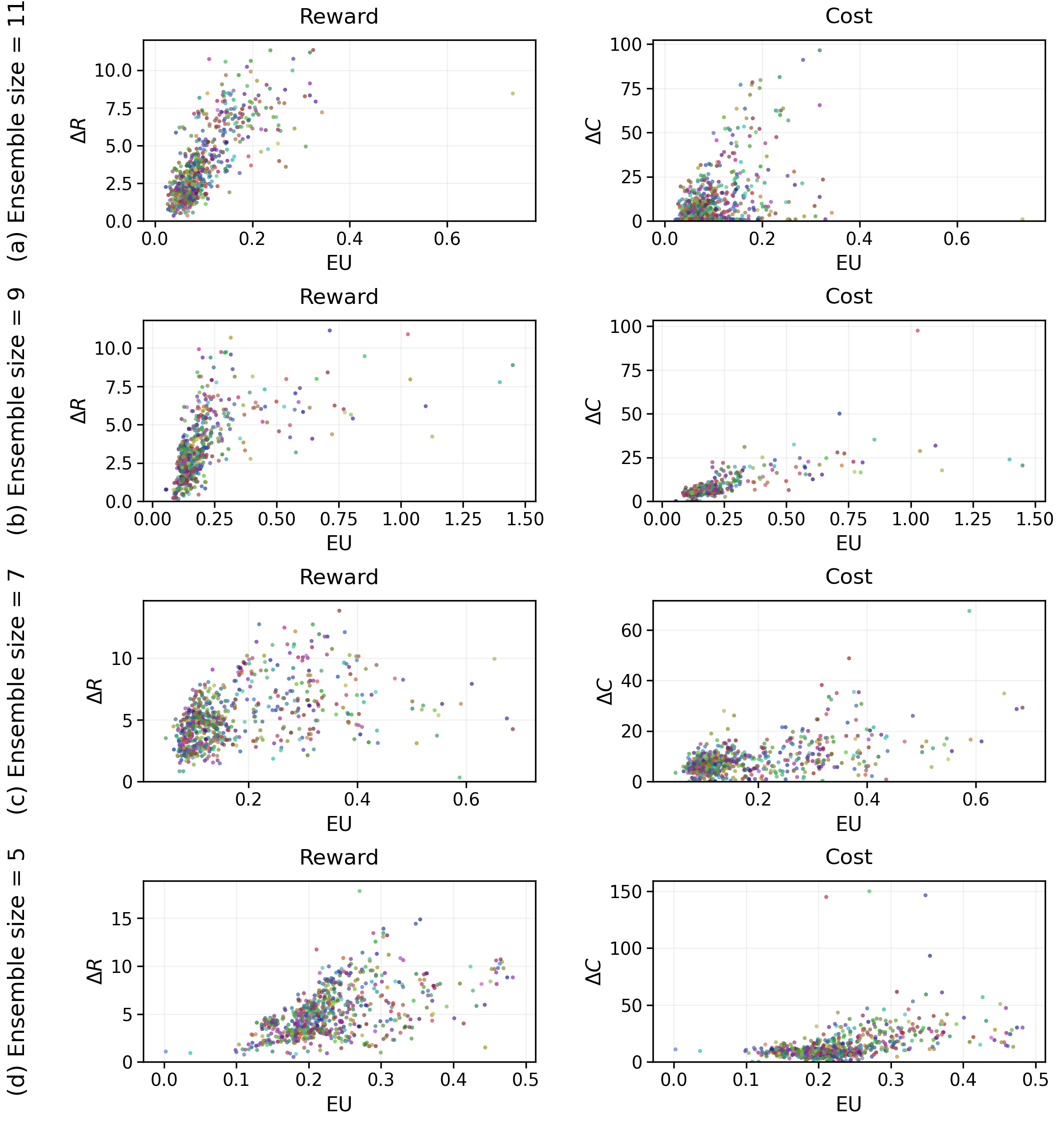}
	\caption{EU evaluation via degradation correlation. 
		Each point corresponds to one held-out episode. The horizontal axis is episode-level EU. 
		The vertical axes
		report the episode-level reward gap \(\Delta R\) and constraint-cost gap
		\(\Delta C\). Larger EU is associated with larger operational
		degradation, which supports its use as an epistemic reliability score for
		quantifying harmful distribution shifts.}
	\label{fig:eu_ood_detection}
\end{figure} 
Fig.~\ref{fig:relative_eu_distribution} further examines how the EU score changes as degradation severity increases. Since raw EU magnitudes are model-dependent, we convert EU values into within-model relative percentiles before pooling the selected models. 
The resulting distributions show a clear upward shift with both reward degradation and cost degradation, indicating that larger OOD
performance deterioration is associated with higher EU.
\begin{figure}[H]
	\centering
	\includegraphics[width=0.98\columnwidth]{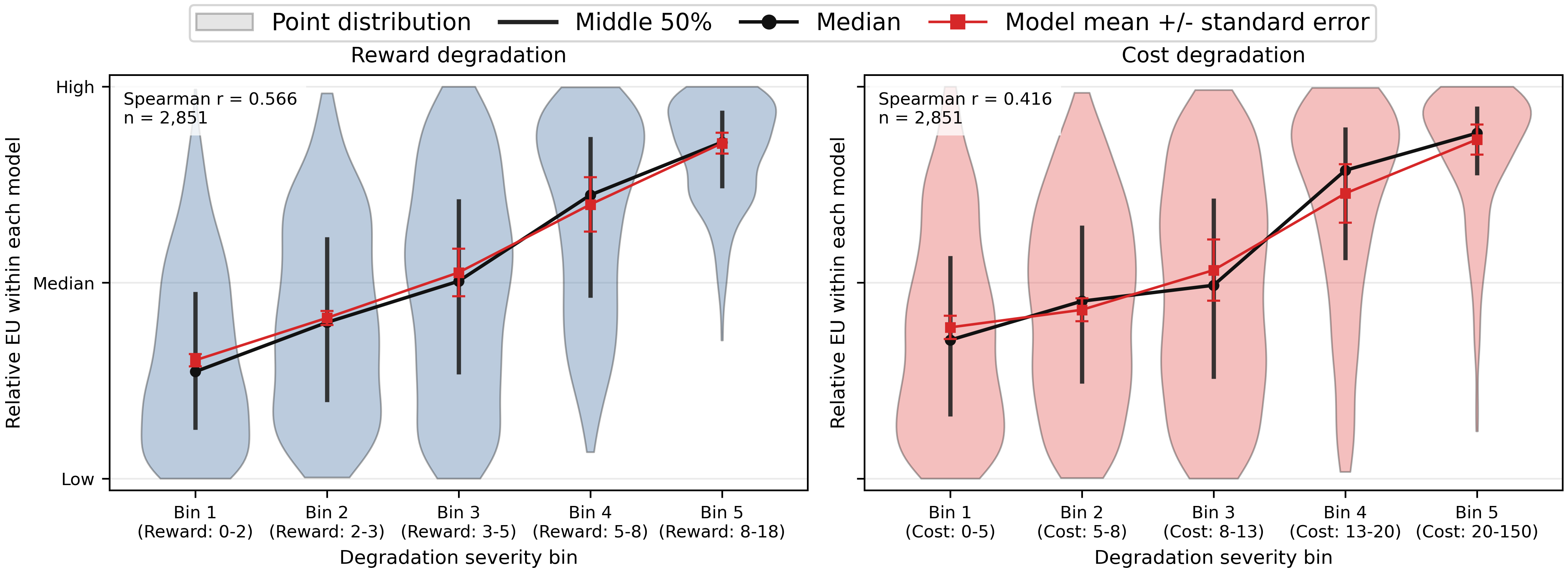}
	 \caption{
	 	Relative EU distributions across degradation severity bins.
	 	Raw EU values are normalized to within-model percentiles before pooling selected models.
	 	Reward and cost degradation are binned into five data-driven severity levels.
	 	Violin plots show point-level distributions; black markers show medians and interquartile ranges; red markers show model means with standard errors.
	 }
	\label{fig:relative_eu_distribution}
\end{figure} 
Table~\ref{tab:degradation_rank_actor_feature} evaluates whether each score
ranks harmful OOD episodes. The actor-feature baselines are unsupervised
InD-fitted novelty scores: GMM/KDE use feature-density NLL, Mahal. uses
Mahalanobis distance, and kNN uses nearest-neighbor distance.  
Like our EU score, these baselines are computed without degradation labels.
For fairness, all scores are
computed per trained model and converted to within-model quantile ranks before
reporting Spearman correlations with \(\Delta R,\Delta C\) and \(\mathrm{P@20}\)
for the top-20\% degraded episodes.
\begin{table}[h]
	\centering
	\caption{Degradation-ranking comparison.}
	\label{tab:degradation_rank_actor_feature}
	\footnotesize
	\setlength{\tabcolsep}{4pt}
	\begin{tabular}{lcccc}
		\toprule
		Score & Spearman \(\Delta R\) & Spearman \(\Delta C\)
		& \(\mathrm{P@20}\) \(\Delta R\) & \(\mathrm{P@20}\) \(\Delta C\) \\
		\midrule
		Ours EU
		& \(0.61 \pm 0.09\)
		& \(0.46 \pm 0.12\)
		& \(0.63 \pm 0.13\)
		& \(0.58 \pm 0.11\) \\
		GMM
		& \(-0.22 \pm 0.19\)
		& \(0.04 \pm 0.36\)
		& \(0.28 \pm 0.16\)
		& \(0.32 \pm 0.21\) \\
		KDE
		& \(-0.17 \pm 0.21\)
		& \(0.11 \pm 0.30\)
		& \(0.28 \pm 0.17\)
		& \(0.35 \pm 0.19\) \\
		Mahal.
		& \(-0.16 \pm 0.21\)
		& \(0.08 \pm 0.35\)
		& \(0.29 \pm 0.17\)
		& \(0.36 \pm 0.20\) \\
		kNN
		& \(-0.17 \pm 0.21\)
		& \(0.09 \pm 0.34\)
		& \(0.29 \pm 0.16\)
		& \(0.36 \pm 0.21\) \\
		\bottomrule
	\end{tabular}
\end{table}

\subsubsection{Closed-Loop Evaluation With EU-Triggered Fallback}

Finally, we evaluate the deployment rule in \eqref{eq:fallback_rule}. We
instantiate \(\mathcal D_{\mathrm{cal}}^{\mathrm{crit}}\) using critical
calibration states selected from the synthetic profile-shift scenarios
described above, including their combinations, which perturb daily load, PV,
price, and net-load trajectories and are disjoint from the held-out deployment
episodes. A state is included in this calibration set when its 
\(\Delta R\) and \(\Delta C\) gap exceeds the acceptable-state tolerance by a
factor of three. For each deployed model, the fallback threshold
\(\tau_{\mathrm{fb}}\) is calibrated from the lower
\(\epsilon_{\mathrm{miss}}\)-quantile of the critical-state EU scores in this
calibration set, following \eqref{eq:safety_threshold}. 
The calibration is
performed per model because raw EU scales are not comparable across
independently trained ensembles.

\begin{figure}[hb]
	\centering
	\includegraphics[width=0.9\columnwidth]{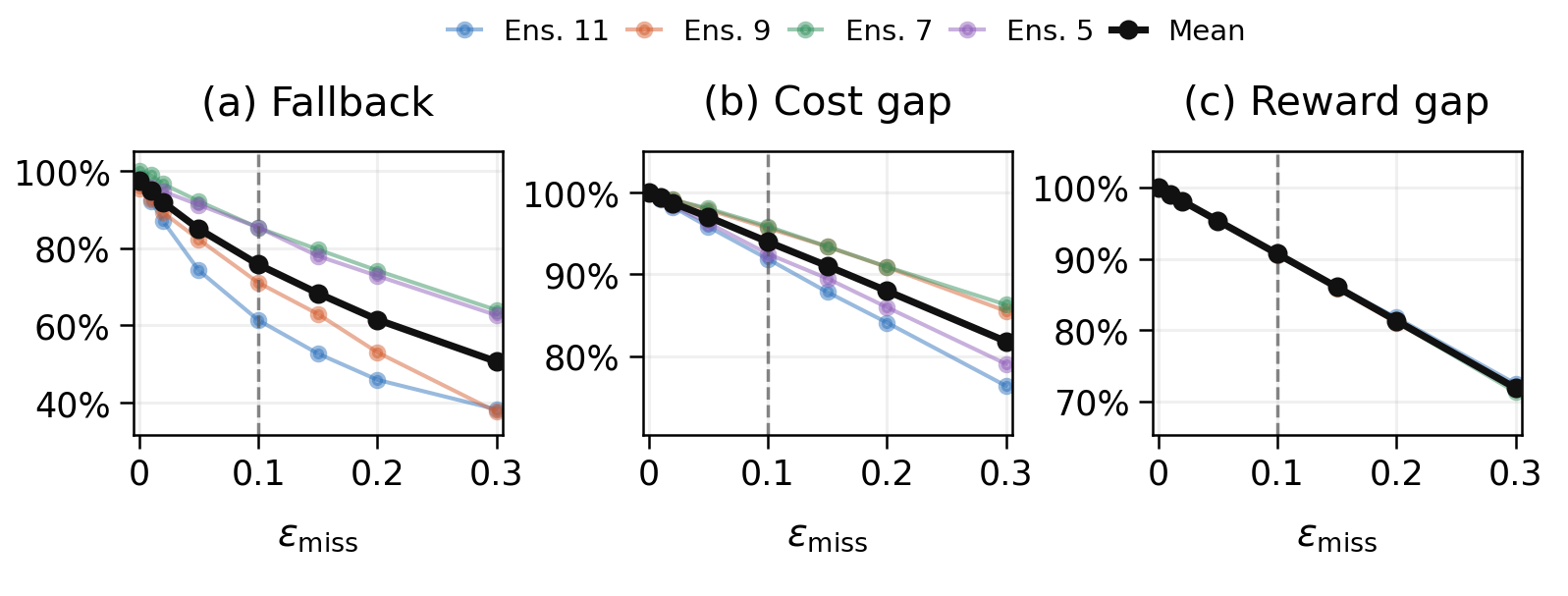}
	\caption{Closed-loop fallback ablation over
		\(\epsilon_{\mathrm{miss}}\). Panel (a) reports fallback invocation;
		panels (b) and (c) report the critical-state \(\Delta C\) and \(\Delta R\) gaps removed by fallback. Colors indicate ensemble size; black denotes
		the mean.}
	\label{fig:eu_fallback_ablation}
\end{figure}

Fig.~\ref{fig:eu_fallback_ablation} summarizes the fallback-rate/gap-removal
tradeoff as \(\epsilon_{\mathrm{miss}}\) varies. At \(\epsilon_{\mathrm{miss}}=0.10\), fallback is invoked for \(75.8\%\) of
held-out episodes and removes \(94.0\%\) of the critical constraint-cost gap
and \(90.7\%\) of the reward gap.
At \(\epsilon_{\mathrm{miss}}=0.20\), fallback
use drops to \(61.5\%\) while constraint-cost gap removal remains \(88.0\%\),
showing that EU-triggered fallback can reduce closed-loop cost with targeted
MISOCP handoffs.
To avoid relying solely on a fixed severe-OOD calibration threshold, we also
evaluate a threshold-free safety--intervention tradeoff in
Fig.~\ref{fig:safety_intervention_tradeoff}. OOD episodes are ranked by EU, and
the fallback rate is swept by assigning the highest-EU fraction to the MISOCP
controller. Increasing fallback coverage consistently reduces both cost
exceedance probability and the 95th-percentile constraint cost, indicating that
EU prioritizes episodes with larger empirical safety risk.

\begin{figure*}[!t]
	\centering
	\includegraphics[width=\textwidth]{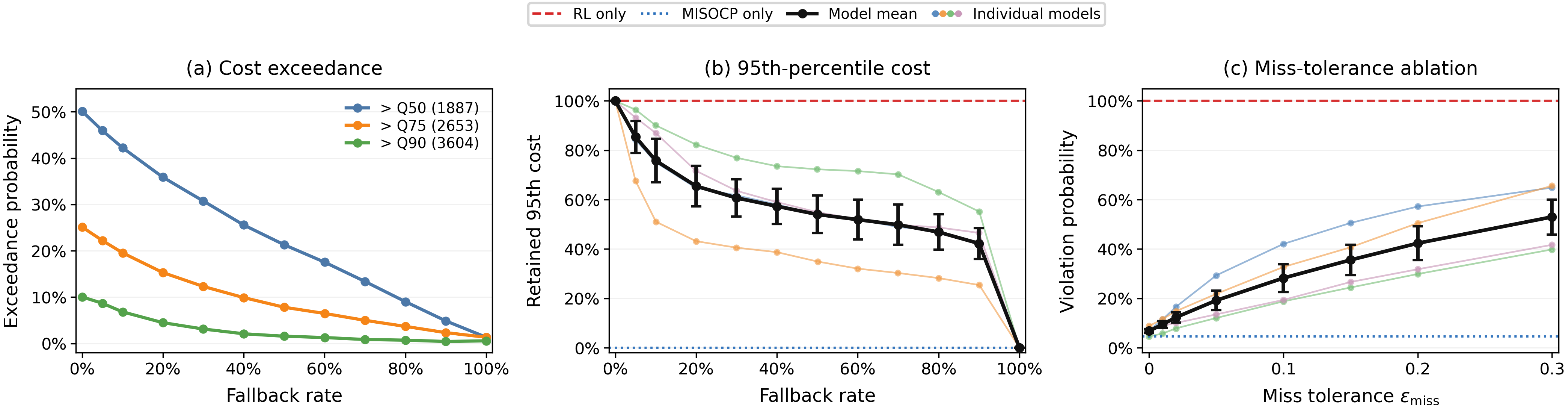}
	
	\caption{Empirical safety--intervention tradeoff of EU-triggered fallback.
		Panels (a) and (b) use a threshold-free evaluation: OOD episodes are ranked by
		EU, and the fallback rate is swept by assigning the highest-EU fraction of
		episodes to the MISOCP controller while leaving the remaining episodes
		controlled by the RL policy.
		Panel (a) reports the probability that the resulting episode constraint cost
		exceeds reference quantiles of the RL-only cost distribution. The numbers in
		parentheses are the corresponding RL-only constraint-cost quantile values; for
		example, \(>Q50(1887)\) means that the deployed controller has episode
		constraint cost larger than the RL-only median cost of 1887.
		Thus, lower curves indicate that EU-guided fallback preferentially removes
		episodes from the high-cost tail.
		Panel (b) reports the retained 95th-percentile constraint cost, normalized by
		the RL-only 95th-percentile cost.
		Panel (c) reports the \(\epsilon_{\mathrm{miss}}\)-based calibration ablation.
		Colored thin curves denote individual selected models, and black curves denote
		the model mean.}
 
	\label{fig:safety_intervention_tradeoff}
\end{figure*}

\section{Conclusion}\label{sec:conclusion}
We proposed an uncertainty-aware DRL framework for safe distribution-network operation by coupling distributional RL with UQ. The resulting second-order UQ decomposes predictive uncertainty into aleatoric and epistemic OOD uncertainty, enabling EU-triggered fallback control with improved constraint compliance in diverse operating conditions.

\appendices
\section{Proofs for Second-order UQ Constructions}
\label{app:uq_proofs} 
 
We adopt the same setting and notation as in Section~\ref{sec:uq_return_critic}.
In particular, the return space is one-dimensional.
On $\mathcal P(Y)$ we use the first-order 2-Wasserstein distance $\Wtwoone$
(induced by the squared Euclidean cost on $Y$), and on
$\mathcal O(Y)=\mathcal P(\mathcal P(Y))$ we use the induced second-order
2-Wasserstein distance $\Wtwotwo$ with ground metric $d_1=\Wtwoone$.

\subsection{Auxiliary Results on Wasserstein Distance in One Dimension}

We first record several standard lemmas on Wasserstein distances and
barycenters in one dimension~\cite{ref29,refPanaretosZemel}.

\begin{lemma}[Dirac second-order measures and couplings]
	\label{lem:dirac_coupling}
	For any $Q\in\mathcal O(Y)$ and $p\in\mathcal P(Y)$ such that
	$\mathbb E_{P\sim Q}\bigl[\bigl(\Wtwoone(P,p)\bigr)^2\bigr]<\infty$,
	$\bigl(\Wtwotwo(Q,\delta_p)\bigr)^2=\mathbb E_{P\sim Q}\bigl[\bigl(\Wtwoone(P,p)\bigr)^2\bigr]$,
	where $\delta_p$ denotes the Dirac mass at $p$ in $\mathcal O(Y)$.
	This follows from
	\cite{ref27}.
\end{lemma}

\begin{lemma}[Quantile representation of $\Wtwoone$ in one dimension~\cite{ref29}]
	\label{lem:w2_quantile}
	Let $Y\subset\mathbb R$, and let $p,\tilde p\in\mathcal P(Y)$ with quantile functions
	$Q_p,Q_{\tilde p}:(0,1)\to\mathbb R$. Then
	$\bigl(\Wtwoone(p,\tilde p)\bigr)^2=\int_0^1 \bigl[Q_p(\tau)-Q_{\tilde p}(\tau)\bigr]^2\,\mathrm d\tau$.
\end{lemma}

\begin{lemma}[One-dimensional Wasserstein barycenter~\cite{refAguehCarlier,ref29}]
	\label{lem:1d_barycenter_quantile}
	Let $p_1,\dots,p_R\in\mathcal P(Y)$ with quantile functions $Q_{p_r}$.
	Consider $p^\star \in \argmin_{p\in\mathcal P(Y)} \frac{1}{R}\sum_{r=1}^R \bigl(\Wtwoone(p_r,p)\bigr)^2$.
	Then a minimiser exists and any minimiser has quantile function
	$Q_{p^\star}(\tau)=\frac{1}{R}\sum_{r=1}^R Q_{p_r}(\tau)$ for $\tau\in(0,1)$.
\end{lemma}

\begin{lemma}[Projection onto Dirac measures and variance~\cite{refPanaretosZemel}]
	\label{lem:w2_dirac_variance}
	Let $p\in\mathcal P(Y)$ have finite second moment and mean $\mu(p)$.
	Then $\inf_{y\in Y}\bigl(\Wtwoone(p,\delta_y)\bigr)^2
	=\bigl(\Wtwoone(p,\delta_{\mu(p)})\bigr)^2
	=\mathrm{Var}_{Y\sim p}(Y)$,
	where $\delta_y$ is the Dirac mass at $y$.
\end{lemma}

\subsection{Return-based Critic}
\label{app:uq_proofs_cont}

We now derive the critic indices in
Section~\ref{sec:uq_return_critic} from the general distance-based definitions
\eqref{eq:utot_def}--\eqref{eq:uep_def}, using the one-dimensional results in
Section~\ref{app:uq_proofs}. 

\subsubsection{Epistemic component as Wasserstein Fr\'echet variance}
 
We start from the epistemic index $U_{\mathrm{ep}}(Q)$ in \eqref{eq:uep_def}
with the choices $d_1=\Wtwoone$ on $\mathcal P(Y)$ and the induced second-order
metric $\Wtwotwo$ on $\mathcal O(Y)$. For any $Q\in\mathcal O(Y)$,
$U_{\mathrm{ep}}(Q)=\inf_{\delta_p \in \mathcal S_{\mathrm{ep}}}\Wtwotwo\bigl(Q,\delta_p\bigr)
=\inf_{p \in \mathcal P(Y)}\Wtwotwo\bigl(Q,\delta_p\bigr)$.
By Lemma~\ref{lem:dirac_coupling}, $\bigl(\Wtwotwo(Q,\delta_p)\bigr)^2
=\mathbb E_{P \sim Q}\bigl[\bigl(\Wtwoone(P,p)\bigr)^2\bigr]$.
Thus, up to a square root, the epistemic index is given by the Fr\'echet
functional on $(\mathcal P(Y),\Wtwoone)$:
$U_{\mathrm{ep}}^{2}(Q):=\inf_{p \in \mathcal P(Y)}\mathbb E_{P \sim Q}\bigl[\bigl(\Wtwoone(P,p)\bigr)^2\bigr]$.
We use $U_{\mathrm{ep}}^{2}(Q)$ as the basic epistemic quantity; it is a
monotone function of $U_{\mathrm{ep}}(Q)$ and preserves the ordering of
epistemic uncertainty levels.
Specializing $U_{\mathrm{ep}}^{2}(Q)$ to the empirical law
$\widehat Q_B(s,a)$ in \eqref{eq:q_critic} implies that any minimizer
is a $\Wtwoone$-barycenter of $\{P_b\}_{b=1}^B$,
which in one dimension admits the closed-form quantile averaging in
Lemma~\ref{lem:1d_barycenter_quantile}; hence the resulting empirical Fr\'echet variance
coincides with $\mathrm{EU}_{\mathrm{crit}}(s,a)$ in \eqref{eq:eu_critic}.

\subsubsection{Aleatoric component as expected intrinsic variance}

Starting from the aleatoric definition \eqref{eq:ual_def} with the reference family
$\mathcal S_{\mathrm{al}}$, we quantify the distance of $Q$ to mixtures of Dirac predictors.
For the derivations below, we work with the squared 2-Wasserstein distance on $\mathcal O(Y)$.
In the one-dimensional setting of Section~\ref{sec:uq_return_critic}, this projection admits
an exact characterization through pointwise projection of each first-order draw $P\sim Q$ onto
Dirac measures. The resulting expression reduces the aleatoric index to the expected intrinsic
variance of the first-order distributions, which leads directly to the closed-form estimator
used in \eqref{eq:au_critic}. 
\begin{lemma}[Exact expression via projection to Dirac predictors]
	\label{lem:ual_upper_bound}
	Let $Q \in \mathcal O(Y)$ be any second-order distribution. Define
	\(\widetilde U_{\mathrm{al}}^2(Q)
	=
	\mathbb E_{P \sim Q}\!\Bigl[
	\inf_{y \in Y}
	\bigl(\Wtwoone(P,\delta_y)\bigr)^2
	\Bigr]\).
	Then
	\(U_{\mathrm{al}}^2(Q)
	:=
	\inf_{\delta_m \in \mathcal S_{\mathrm{al}}}
	\bigl(\Wtwotwo(Q,\delta_m)\bigr)^2
	=
	\widetilde U_{\mathrm{al}}^2(Q)\).
\end{lemma}

\begin{IEEEproof}
	We prove the two inequalities.
	
	\paragraph*{Step 1} We show that
	$U_{\mathrm{al}}^{2}(Q) \le \widetilde{U}_{\mathrm{al}}^{2}(Q)$.
	For each $P \in \mathcal P(Y)$ with finite second moment, take
	$y^\star(P):=\mu(P)$, the mean of $P$.
	By Lemma~\ref{lem:w2_dirac_variance}, $y^\star(P)$ is an optimizer of
	$\inf_{y\in Y} \bigl(\Wtwoone(P,\delta_y)\bigr)^2$, and
	\(\inf_{y\in Y} \bigl(\Wtwoone(P,\delta_y)\bigr)^2
	=
	\bigl(\Wtwoone(P,\delta_{\mu(P)})\bigr)^2\).
	Define the pushforward measure $m := (y^\star)_\# Q \in \mathcal P(Y)$ and
	consider $\delta_m \in \mathcal S_{\mathrm{al}}$.
	Define $T:\mathcal P(Y)\to \mathcal P(Y)$ by
	$T(P):=\delta_{y^\star(P)}=\delta_{\mu(P)}$,
	and construct the coupling
	$\gamma := (\mathrm{Id},T)_\# Q \in \Gamma(Q,\delta_m)$.
	Then, by the definition of the second-order distance $\Wtwotwo$ on $\mathcal O(Y)$
	with ground metric $\Wtwoone$ on $\mathcal P(Y)$,
	\(\bigl(\Wtwotwo(Q,\delta_m)\bigr)^2
	\le
	\int_{\mathcal P(Y)\times\mathcal P(Y)}
	\bigl(\Wtwoone(P,\tilde P)\bigr)^2
	\,\mathrm d\gamma(P,\tilde P)
	=
	\mathbb E_{P\sim Q}
	\Bigl[
	\bigl(\Wtwoone(P,T(P))\bigr)^2
	\Bigr]\).
	Thus,
	\(\bigl(\Wtwotwo(Q,\delta_m)\bigr)^2
	\le
	\mathbb E_{P\sim Q}
	\Bigl[
	\bigl(\Wtwoone(P,\delta_{\mu(P)})\bigr)^2
	\Bigr]
	=
	\widetilde U_{\mathrm{al}}^2(Q)\).
	Taking the infimum over $\delta_m \in \mathcal S_{\mathrm{al}}$ yields
	$U_{\mathrm{al}}^{2}(Q)\le \widetilde U_{\mathrm{al}}^{2}(Q)$.
	
	\paragraph*{Step 2} We show that $U_{\mathrm{al}}^2(Q) \ge \widetilde U_{\mathrm{al}}^2(Q)$.
	Fix any $m \in \mathcal P(Y)$ and any coupling $\gamma \in \Gamma(Q,\delta_m)$.
	Since $\delta_m$ is supported on Dirac predictors $\delta_y$, we have
	$\tilde P=\delta_y$ $\gamma$-almost surely for some $y\in Y$.
	Hence, for $\gamma$-almost every $(P,\tilde P)$,
	\(\bigl(\Wtwoone(P,\tilde P)\bigr)^2
	=
	\bigl(\Wtwoone(P,\delta_y)\bigr)^2
	\ge
	\inf_{y' \in Y}
	\bigl(\Wtwoone(P,\delta_{y'})\bigr)^2\).
	Integrating both sides with respect to $\gamma$ and using that the first marginal
	of $\gamma$ is $Q$, we obtain
	\(\int_{\mathcal P(Y)\times\mathcal P(Y)}
	\bigl(\Wtwoone(P,\tilde P)\bigr)^2
	\,\mathrm d\gamma(P,\tilde P)
	\ge
	\mathbb E_{P\sim Q}
	\Bigl[
	\inf_{y' \in Y}
	\bigl(\Wtwoone(P,\delta_{y'})\bigr)^2
	\Bigr]
	=
	\widetilde U_{\mathrm{al}}^2(Q)\).
	Taking the infimum over all couplings $\gamma \in \Gamma(Q,\delta_m)$ yields
	\(\bigl(\Wtwotwo(Q,\delta_m)\bigr)^2
	\ge
	\widetilde U_{\mathrm{al}}^2(Q)\).
	Finally, taking the infimum over $\delta_m \in \mathcal S_{\mathrm{al}}$ gives
	$U_{\mathrm{al}}^2(Q)\ge \widetilde U_{\mathrm{al}}^2(Q)$.
	
	Combining Step 1 and Step 2 concludes that
	\(U_{\mathrm{al}}^2(Q)=\widetilde U_{\mathrm{al}}^2(Q)\).
\end{IEEEproof}

Combining Lemma~\ref{lem:ual_upper_bound} with
Lemma~\ref{lem:w2_dirac_variance} yields an explicit expression for
\(\widetilde U_{\mathrm{al}}^2(Q)\) in terms of intrinsic variances.

\begin{corollary}[Expected intrinsic variance]
	\label{cor:ual_variance}
	Let \(Q \in \mathcal O(Y)\) and let \(P \sim Q\) have finite second
	moments almost surely. Then
	$
		\widetilde U_{\mathrm{al}}^2(Q)
		=
		\mathbb E_{P \sim Q}
		\Bigl[
		\mathrm{Var}_{Y \sim P}(Y)
		\Bigr].
	$
\end{corollary}

We take \(\widetilde U_{\mathrm{al}}^2(Q)\) as the basic aleatoric
quantity; by Lemma~\ref{lem:ual_upper_bound} it equals
\(U_{\mathrm{al}}^2(Q)\) and thus is consistent with the distance-based
interpretation. 
Applying Corollary~\ref{cor:ual_variance} to $\widehat Q_B(s,a)$ in \eqref{eq:q_critic} gives
$
U_{\mathrm{al}}^2\!\bigl(\widehat Q_B(s,a)\bigr)
=
\frac{1}{B}\sum_{b=1}^{B}\mathrm{Var}_{Y\sim P_b}(Y),$
which is exactly $\mathrm{AU}_{\mathrm{crit}}(s,a)$ in \eqref{eq:au_critic}.

%
%

\section{Consistency of Second-order UQ Estimators}
\label{app:eu_au_consistency} 
In implementation, each draw $P_b$ is obtained by independently sampling an ensemble member and a dropout mask (and the IQN quantile fractions), which yields i.i.d.\ draws from $Q_{\mathrm{crit}}(s,a)$.
The corresponding empirical second-order measure is $\widehat Q_B(s,a)$ in \eqref{eq:q_critic}.

\begin{theorem}[Consistency of Monte-Carlo EU/AU estimators]
	\label{thm:consistency_eu_au}
 
    Define the population indices by
	$
	\mathrm{EU}^\star_{\mathrm{crit}}(s,a):=\mathbb E\!\left[(\Wtwoone(P,p^\star))^2\right], 
	$ $ 
	\mathrm{AU}^\star_{\mathrm{crit}}(s,a):=\mathbb E\!\left[\mathrm{Var}_{Y\sim P}(Y)\right].
	\label{eq:pop_eu_au_app}
	$
	Then, as $B\to\infty$,
	$
	\widehat{\mathrm{EU}}_{\mathrm{crit},B}(s,a)
	$
	$\xrightarrow{\mathrm{almost\, surely}}$
	$ \mathrm{EU}^\star_{\mathrm{crit}}(s,a),
	$
	$
	\widehat{\mathrm{AU}}_{\mathrm{crit},B}(s,a)
	$
	$\xrightarrow{\mathrm{almost\, surely}} \mathrm{AU}^\star_{\mathrm{crit}}(s,a).
	\label{eq:eu_au_consistency_app}
	$
\end{theorem}

\begin{IEEEproof}
	Let $\{P_b\}_{b=1}^B$ be i.i.d.\ from $Q_{\mathrm{crit}}(s,a)$ and let $\widehat p_B$ be an empirical barycenter as in \eqref{eq:critic_barycenter}. By the existence and strong consistency of empirical Wasserstein barycenters on $\mathcal P_2(Y)$ \cite{refW2BarycenterConsistency}, we have $\Wtwoone(\widehat p_B,p^\star)\to 0$ almost surely. For the AU part, from \eqref{eq:au_critic},
	$\widehat{\mathrm{AU}}_{\mathrm{crit},B}(s,a)=\frac{1}{B}\sum_{b=1}^B \mathrm{Var}_{Y\sim P_b}(Y)$.
	Since $P\in\mathcal P_2(Y)$ almost surely, $\mathrm{Var}_{Y\sim P}(Y)<\infty$ almost surely and is integrable; thus by the strong law of large numbers,
	$\widehat{\mathrm{AU}}_{\mathrm{crit},B}(s,a)\to \mathbb E[\mathrm{Var}_{Y\sim P}(Y)]=\mathrm{AU}^\star_{\mathrm{crit}}(s,a)$ almost surely.
	For the EU part, by \eqref{eq:eu_critic},
	$\widehat{\mathrm{EU}}_{\mathrm{crit},B}(s,a)=\frac{1}{B}\sum_{b=1}^B (\Wtwoone(P_b,\widehat p_B))^2$.
	Add and subtract the population barycenter:
	\begin{align}
		&\widehat{\mathrm{EU}}_{\mathrm{crit},B}(s,a)-\mathrm{EU}^\star_{\mathrm{crit}}(s,a) \nonumber
		\\&\quad=
		\underbrace{\Bigl(\tfrac{1}{B}\sum_{b=1}^B (\Wtwoone(P_b,p^\star))^2-\mathbb E[(\Wtwoone(P,p^\star))^2]\Bigr)}_{(I)}
		\nonumber\\
		&\quad\quad+
		\underbrace{\tfrac{1}{B}\sum_{b=1}^B\Bigl((\Wtwoone(P_b,\widehat p_B))^2-(\Wtwoone(P_b,p^\star))^2\Bigr)}_{(II)}.
		\label{eq:eu_decomp_app}
	\end{align}
	Term $(I)\to 0$ almost surely by the strong law since $(\Wtwoone(P,p^\star))^2$ is integrable under $P\in\mathcal P_2(Y)$. 
	For $(II)$, using $|x^2-y^2|\le (x+y)|x-y|$ with $x=\Wtwoone(P_b,\widehat p_B)$ and $y=\Wtwoone(P_b,p^\star)$, then the triangle inequality gives
	\begin{align}
		&\bigl|(\Wtwoone(P_b,\widehat p_B))^2-(\Wtwoone(P_b,p^\star))^2\bigr| \nonumber\\
		&\quad \le 
		\bigl(2\Wtwoone(P_b,p^\star)+\Wtwoone(\widehat p_B,p^\star)\bigr)\,\Wtwoone(\widehat p_B,p^\star).
		\label{eq:eu_pointwise_bound_app}
	\end{align}
	Averaging \eqref{eq:eu_pointwise_bound_app} over $b$ yields
	\begin{equation}
		|(II)|
		\le
		\Bigl(2\cdot\tfrac{1}{B}\sum_{b=1}^B \Wtwoone(P_b,p^\star)+\Wtwoone(\widehat p_B,p^\star)\Bigr)\,\Wtwoone(\widehat p_B,p^\star).
		\label{eq:eu_avg_bound_app}
	\end{equation}
	Moreover, $P\in\mathcal P_2(Y)$ implies $\mathbb E[\Wtwoone(P,p^\star)]<\infty$, hence by the strong law
	$\tfrac{1}{B}\sum_{b=1}^B \Wtwoone(P_b,p^\star)\to \mathbb E[\Wtwoone(P,p^\star)]$ almost surely, while $\Wtwoone(\widehat p_B,p^\star)\to 0$ almost surely.
	Therefore the right-hand side of \eqref{eq:eu_avg_bound_app} converges to $0$ almost surely, so $(II)\to 0$ almost surely.
	Combining $(I)$ and $(II)$ in \eqref{eq:eu_decomp_app} proves
	$\widehat{\mathrm{EU}}_{\mathrm{crit},B}(s,a)\to \mathrm{EU}^\star_{\mathrm{crit}}(s,a)$ almost surely.
\end{IEEEproof}

\vfill

\end{document}